\documentclass[review]{elsarticle}

\usepackage{amssymb}
\usepackage{amsthm}
\usepackage{amsmath}
\usepackage{subfigure}
\usepackage{hyperref}
\usepackage{algorithm}
\usepackage{algpseudocode}
\usepackage{booktabs}
\usepackage{caption}
\algdef{SE}[DOWHILE]{Do}{doWhile}{\algorithmicdo}[1]{\algorithmicwhile\ #1}%
\algtext*{EndWhile}
\algtext*{EndIf}
\algtext*{EndFor}
\algtext*{EndFunction}

\usepackage{url}

\newtheorem{lemma}{Lemma}
\newtheorem{theorem}{Theorem}
\newtheorem{definition}{Definition}
\newtheorem{corollary}{Corollary}
\newtheorem{proposition}{Proposition}

\algrenewcommand\algorithmicindent{0.8em}%

\journal{Pattern Recognition}

\DeclareCaptionFormat{algor}{%
  \hrulefill\par\offinterlineskip\vskip1pt%
    \textbf{#1#2}#3\offinterlineskip\hrulefill}
\DeclareCaptionStyle{algori}{singlelinecheck=off,format=algor,labelsep=space}
\begin{document}

\begin{frontmatter}
\title{Approximately Optimal Binning for the Piecewise Constant Approximation of the Normalized Unexplained Variance (nUV) Dissimilarity Measure}

\author[1]{Attila Fazekas\corref{cor1}%
}
\ead{attila.fazekas@inf.unideb.hu}
\author[2]{Gy\"orgy Kov\'acs}
\ead{gyuriofkovacs@gmail.com}

\cortext[cor1]{Corresponding author}
\address[1]{University of Debrecen, Faculty of Informatics, PO Box 400, Debrecen 4002, Hungary}
\address[2]{Analytical Minds Ltd., \'Arp\'ad street 5, Beregsur\'any 4933, Hungary}




\begin{abstract}
The recently introduced Matching by Tone Mapping (MTM) dissimilarity measure enables template matching under smooth non-linear distortions and also has a well-established mathematical background. MTM operates by binning the template, but the ideal binning for a particular problem is an open question. By pointing out an important analogy between the well known mutual information (MI) and MTM, we introduce the term \emph{normalized unexplained variance} (nUV) for MTM to emphasize its relevance and applicability beyond image processing. Then, we provide theoretical results on the optimal binning technique for the nUV measure and propose algorithms to find approximate solutions. The theoretical findings are supported by numerical experiments. Using the proposed techniques for binning shows 4-13\% increase in terms of AUC scores with statistical significance, enabling us to conclude that the proposed binning techniques have the potential to improve the performance of the nUV measure in real applications.
\end{abstract}

\begin{keyword}
dissimilarity\sep template matching\sep matching by tone mapping\sep optimal binning\sep normalized unexplained variance
\MSC[2010] 41A05\sep 41A10\sep 65D05\sep 65D18  
\end{keyword}


\end{frontmatter}


\section{Introduction}

One of the most general goals of pattern recognition is to distinguish noisy and/or distorted realizations of some patterns of interest (POI) from realizations of other patterns and noise. A common point of most approaches is that they explicitly or implicitly define/quantify the problem-specific notion of \emph{similarity} (or the inversely proportional dissimilarity) of patterns, usually through (dis)similarity measures. The role of (dis)similarity measures depends on how the available knowledge and the POI are represented, by labeled datasets (leading to machine learning approaches) or by exact models (leading to template matching).

\emph{Machine learning approaches.} If there is a hand-labeled dataset containing multiple realizations of the POI and counterexamples, one can exploit general-purpose machine learning techniques to address recognition \cite{deeprecognition} tasks directly. Although some techniques like \emph{metric learning} \cite{metriclearning} approaches learn the most suitable (dis)similarity measure explicitly, many of the commonly used regressors and classifiers use relatively simple (dis)similarity measures (for example, \emph{Euclidean distance} in k-Nearest Neighbors (kNN) \cite{machinelearning} and kernel functions in Support Vector Machines (SVM) \cite{machinelearning}).
One of the main reasons why advanced (dis)similarity measures rarely appear in general-purpose machine learning techniques is that they make assumptions on the distribution of the data, and these assumptions are likely to fail in general problems when one has no information about the possible distortions.  
Instead, machine learning techniques learn the application-specific meaning of (dis)similarity from the data and represent the advanced concepts of (dis)similarity in the inner structure of the machine learning model in terms of the simple measures.


\emph{Template matching approaches.} There are numerous problems with no hand-labeled datasets, but one high-quality realization or exact model of the POI. In these cases, one declares the (dis)similarity measure to be used according to the expected distortions of the POI, and considers sufficiently similar patterns as a realization of the POI. This approach is usually referred as \emph{template matching}, exploited in various problems where data is acquired in highly controlled environments with low and/or predictable variability (like quality checking on conveyor belts \cite{industry0}); the POI is simple enough to be represented by one realization or exact model (for example, in medical imaging \cite{segmentation0}); the POI changes by time and the training of pattern-specific solutions is infeasible (like in object tracking applications \cite{tracking0}).

This paper deals with the second class of problems (applications, where (dis)similarity measures invariant to certain types of distortions are needed) and presents some theoretical results related to the recently introduced dissimilarity measure Matching by Tone Mapping (MTM) \cite{mtm}, which was shown to provide superior performance in numerous template matching and even registration scenarios due to its approximate invariance to even non-linear distortions.
Before moving on to the presentation of the findings, we provide a brief overview of the most widely used measures to enable the positioning of the work in the literature of (dis)similarity measures. For the ease of discussion, we introduce the terminology of template matching: let $\mathbf{t}, \mathbf{w}\in\mathbb{R}^d$ denote a template (pattern) and a window of a signal the template is being compared to, respectively. By \emph{intensity transformation} (\emph{tone mapping}/distortion) $\mathcal{M}: \mathbb{R}\rightarrow\mathbb{R}$ we refer to a deterministic function applied to each coordinate of its parameter vector independently by introducing the notation $\mathcal{M}[\mathbf{t}]_i= \mathcal{M}(\mathbf{t}_i)$. We mention that the invariance of a (dis)similarity measure $D$ to distortions like $\mathcal{M}$ ($D(\mathbf{t}, \mathbf{w}) = D(\mathbf{t}, \mathcal{M}[\mathbf{w}])$) is usually referred as \emph{photometric invariance} although the concept is applicable in other fields of signal processing beyond imaging.

As the meaning of (dis)similarity is usually application-specific, numerous measures have been proposed in the last decades. Probably the simplest dissimilarity measures are the $L_p$ distances ($L_1$ and $L_2$ also known as \emph{Manhattan} and \emph{Euclidean} distances) with no invariance to any distortion $\mathcal{M}$. \emph{Cross-correlation} (CC) and \emph{normalized Euclidean distance} \cite{book1} are invariant to scaling, while the \emph{Pearson correlation coefficient} (PCC) is invariant to linear distortions. Although these measures are invariant to linear transformations at most, they usually serve as building blocks of advanced techniques or they are made invariant to certain classes of non-linear transformations by the kernel-trick \cite{kernel}.

Numerous (dis)similarity measures (like \emph{Spearman's Rho} \cite{spearman0} and \emph{Kendall's Tau} \cite{kendall0} ) are based on the rank transformation $R: \mathbb{R}^d\rightarrow \mathbb{Z}^d$, where $R(\mathbf{t})_i$ is the rank of $\mathbf{t}_i$ among $\lbrace \mathbf{t}_i\rbrace_{i\in\lbrace 1, \dots, d\rbrace}$, and use one of the simple measures to quantify the (dis)similarity of $R(\mathbf{t})$ and $R(\mathbf{w})$ instead of $\mathbf{t}$ and $\mathbf{w}$. Although the ranking of elements is not affected by monotonic transformations, and consequently these methods are invariant to monotonic distortions, a common drawback is their sensitivity to noise and ties among the elements of $\mathbf{t}$ and $\mathbf{w}$.


A large and popular family of (dis)similarity functions \cite{shannon0, renyi} is based on information-theoretical concepts by quantifying the \emph{mutual information} (MI) content in the intensity distributions of the template and the window. Alternatively, the comparison of the distributions of derived local quantities, like gradient orientation \cite{shannon1} was also proposed. Although the MI-based measures are considered to be invariant to even non-linear intensity transformations, the estimation of joint densities can be challenging, especially for small templates.

\emph{Correlation ratio} and its variants \cite{cratio1, cratio2} characterize the degree to which $\mathbf{w}$ can be treated as a single-valued function of $\mathbf{t}$, and were shown to provide better performance than MI in certain registration problems \cite{cratio2}.

Invariance to certain distortions can be achieved by extracting invariant features from $\mathbf{t}$ and $\mathbf{w}$ and quantifying the (dis)similarity of the feature vectors. A brief overview of photometric invariant features can be found in \cite{photometric}. Commonly used features in the imaging domain, invariant to certain types of geometric and photometric distortions are \emph{Hu's descriptors} \cite{hu0} (combinations of statistical moments) and \emph{local binary patterns} (LBP) \cite{lbp0} (based on the intensity differences of a pixel and its neighbors). 

Although geometric distortions are out of scope for this paper, we mention that some measures used in the imaging domain are invariant to even affine or projective geometrical transformations \cite{sift}.



For further details on (dis)similarity measures, excellent overviews can be found in the books \cite{book0}, \cite{book1}.

Recently, the \emph{Matching by Tone Mapping} (MTM) \cite{mtm} measure was proposed for photometric invariant template matching and registration and its restriction to monotonic distortions was also introduced \cite{mmtm}. As MTM was shown to give superior performance in numerous template matching and registration scenarios \cite{mtm} and can be computed efficiently in terms of some convolution operations, it has many potential applications in signal processing. Similarly to MI and related techniques (being approximately invariant to non-linear distortions), MTM operates by binning the template, however the proper selection of bins providing optimal performance according to some criteria is still an open question. In this paper, we carry out a statistical analysis of the effect of bin selection for MTM. 


The main contributions of the paper to the field are summarized as follows:
\begin{enumerate}
\item As the name suggests, MTM was developed for image processing, where various distortions of a template can be treated as tone mappings. We point out that MTM is a highly analogous concept to MI, with numerous potential applications beyond imaging. In order to emphasize the generality of the measure, we introduce the name \emph{normalized Unexplained Variance} (nUV) which we found more more conformant with the literature of statistics.
\item We define criteria for the ideal operation of the nUV measure, provide theoretical results on the ideal binning under these criteria and also provide algorithms to determine the ideal binning for particular problems. 
\item By numerical simulations, we show that in the context of discriminating distorted templates from noise, the proposed binning techniques improve the discrimination power of nUV by 4-13\% in terms of AUC scores, with statistical significance.
\end{enumerate}

The paper is organized as follows. In Section \ref{secmtm} a brief introduction is given to MTM, its analogy to MI is pointed out and the new nomenclature of nUV is introduced. The optimality criterion is defined, theoretical results are derived and corresponding algorithms to find an approximately optimal binnings are proposed in Section \ref{secopt}. The numerical experiments are described and evaluated in Section \ref{secres}, and finally, conclusions are drawn in Section \ref{secconc}.

\section{Brief Introduction to Matching by Tone Mapping (MTM) and Problem Formulation}
\label{secmtm}
In this section, we give a brief introduction of the MTM measure, discuss the importance of binning, formulate the problem we deal with in the rest of the paper and also point out a close relation between MTM and MI leading us to the introduction of the term $\emph{normalized Unexplained Variance}$ (nUV). 

First, the notations used in the rest of the paper are introduced, trying to follow those of the related papers \cite{mtm, mmtm} for the compatibility of discussions. We use lowercase, boldface and uppercase letters to denote scalars, vectors and matrices, respectively (e.g. $x\in\mathbb{R}$, $\mathbf{t}\in\mathbb{R}^d$, $S\in\lbrace 0, 1\rbrace^{d\times b}$), keeping the notations $\mathbf{t}$ and $\mathbf{w}$ for the template and the window and $d$ for the dimensionality of the feature space. Sets, and the special class of functions called intensity transformations (distortions or \emph{tone mappings} in \cite{mtm}) are denoted by calligraphic letters like $\mathcal{I}\in\lbrace 1, 2, \dots, d\rbrace$, and $\mathcal{M}:\mathbb{R}\rightarrow\mathbb{R}$, respectively, recalling $\mathcal{M}[\mathbf{t}]_i= \mathcal{M}(\mathbf{t}_i)$. For the ease of reading, and for compatility with literature \cite{mtm}, we also introduce greek letters, which always denote vectors in special roles. 

The MTM dissimilarity \cite{mtm} of $\mathbf{t}$ and $\mathbf{w}$ is defined as
\begin{equation}
\label{eqmtm0}
MTM_{\text{ideal}}(\mathbf{t}, \mathbf{w})= \min\limits_{\mathcal{M}: \mathbb{R}\to\mathbb{R}}\left\lbrace \frac{\Vert \mathcal{M}[\mathbf{t}] - \mathbf{w}\Vert^2}{d \text{var}(\mathbf{w})}\right\rbrace,
\end{equation}
where the numerator measures how close $\mathbf{t}$ can be transformed to $\mathbf{w}$ by applying some tone mapping  $\mathcal{M}$ coordinate-wise and the function $\text{var}(\mathbf{w})$ in the denominator stands for the empirical variance of the elements of $\mathbf{w}$, ensuring invariance to intensity scaling.
It is worth noting that MTM is not symmetric: the form (\ref{eqmtm0}) is referred as the \emph{Pattern-to-Window} (PtW) case and the \emph{Window-to-Pattern} (WtP) is defined by interchanging $\mathbf{t}$ and $\mathbf{w}$ in (\ref{eqmtm0}). In the rest of the paper we focus on the Pattern-to-Window case, but emphasize that all results can be derived for the \emph{Window-to-Pattern} (WtP) case analogously.

\subsection{Piecewise constant approximation}

The minimization problem (\ref{eqmtm0}) cannot be solved explicitly, but approximate solutions can be obtained by the linearization of the problem, particularly, replacing the term $\mathcal{M}(\mathbf{t})$ with a linear approximation. 
Let the coordinates of $\mathbf{t}$ be quantized into $b\in\mathbb{Z}^{+}$ bins and let the boundaries of the bins arranged into the vector $\mathbf{q}\in\mathbb{R}^{b+1}$, supposing that $\mathbf{q}_1 \leq \min_{i}\mathbf{t}_i$, $\max_{i}\mathbf{t}_i < \mathbf{q}_{b+1}$ and each bin $[\mathbf{q}_i, \mathbf{q}_{i+1}[$ contains at least one element. One can form the piecewise constant (PWC) \emph{slice transform} matrix $S(\mathbf{t})\in\lbrace 0, 1\rbrace^{d\times b}$ of $\mathbf{t}$ as
\begin{equation}
S_{ij}= \begin{cases} 1, & \text{if } \mathbf{q}_j \leq \mathbf{t}_i < \mathbf{q}_{j+1}, \\ 0, & \text{otherwise}.\end{cases}
\end{equation}
It can be readily seen that the matrix $S$ contains structural information about $\mathbf{t}$, each column is related to a bin, and the $i$th element of column $j$ is set to $1$ only if $\mathbf{t}_i$ falls in the bin $[\mathbf{q}_j,\mathbf{q}_{j+1}[$. The columns of the matrix $S$ are referred as \emph{slices}, the cardinalities of the slices are represented in the vector $\mathbf{n}\in\mathbb{Z}_+^b$ with $\mathbf{n}_i$ denoting the number of elements falling in slice $i$. 
Given $S$, one can approximate $\mathbf{t}$ 
as $\mathbf{t}\simeq S\mathbf{\boldsymbol\beta}$, $\mathbf{\boldsymbol\beta}\in\mathbb{R}^b$ in many ways, e.g. $\mathbf{\boldsymbol\beta}_j= (\mathbf{q}_j + \mathbf{q}_{j+1})/2$ or $\mathbf{\boldsymbol\beta}_j= \mathbf{q}_j$. 
Similarly to the approximation of $\mathbf{t}$, the matrix $S$ can be used to approximate various coordinate-wise transformations of $\mathbf{t}$, for example, the vector $\mathbf{u}= S\mathbf{\boldsymbol\beta}$, $\mathbf{\boldsymbol\beta}_j=\mathbf{q}_j^2$ can be considered as an approximation of the vector $\mathcal{M}[\mathbf{t}]\in\mathbb{R}^d$ derived from $\mathbf{t}$ by applying the tone mapping $\mathcal{M}(x)= x^2$ coordinate-wise. Analogously, for any $\mathbf{\boldsymbol\beta}\in\mathbb{R}^b$, the expression $S\mathbf{\boldsymbol\beta}$ can be considered as the PWC approximation of some possibly non-linear coordinate-wise transformation of $\mathbf{t}$. 
Obviously, the quality of approximation highly depends on the intensity distribution of $\mathbf{t}$, the number of bins, and the smoothness of $\mathcal{M}$. Nevertheless, the linearization of the minimization problem (\ref{eqmtm0}) by $\mathcal{M}(\mathbf{t})\simeq S\boldsymbol\beta$ is reasonable, and PWC MTM becomes
\begin{align}
\label{mtmlin}
D(\mathbf{t}, \mathbf{w})&= \min\limits_{\boldsymbol\beta\in\mathbb{R}^{b}}\left\lbrace \frac{\Vert S\boldsymbol\beta - \mathbf{w}\Vert^2}{d \text{var}(\mathbf{w})}\right\rbrace= \frac{\Vert S\hat{\boldsymbol\beta} - \mathbf{w}\Vert^2}{d \text{var}(\mathbf{w})},
\end{align}
where $\hat{\boldsymbol\beta}$ is the exact solution of the least squares problem in the numerator:
\begin{align}
\label{hatbeta0}
\hat{\mathbf{\boldsymbol\beta}}= \arg\min\limits_{\mathbf{\boldsymbol\beta}\in\mathbb{R}^b}\Vert S\mathbf{\boldsymbol\beta} - \mathbf{w}\Vert^2= (S^TS)^{-1}S^T\mathbf{w}.
\end{align}
\begin{figure}[t]
\begin{center}
    \includegraphics[width=0.7\textwidth]{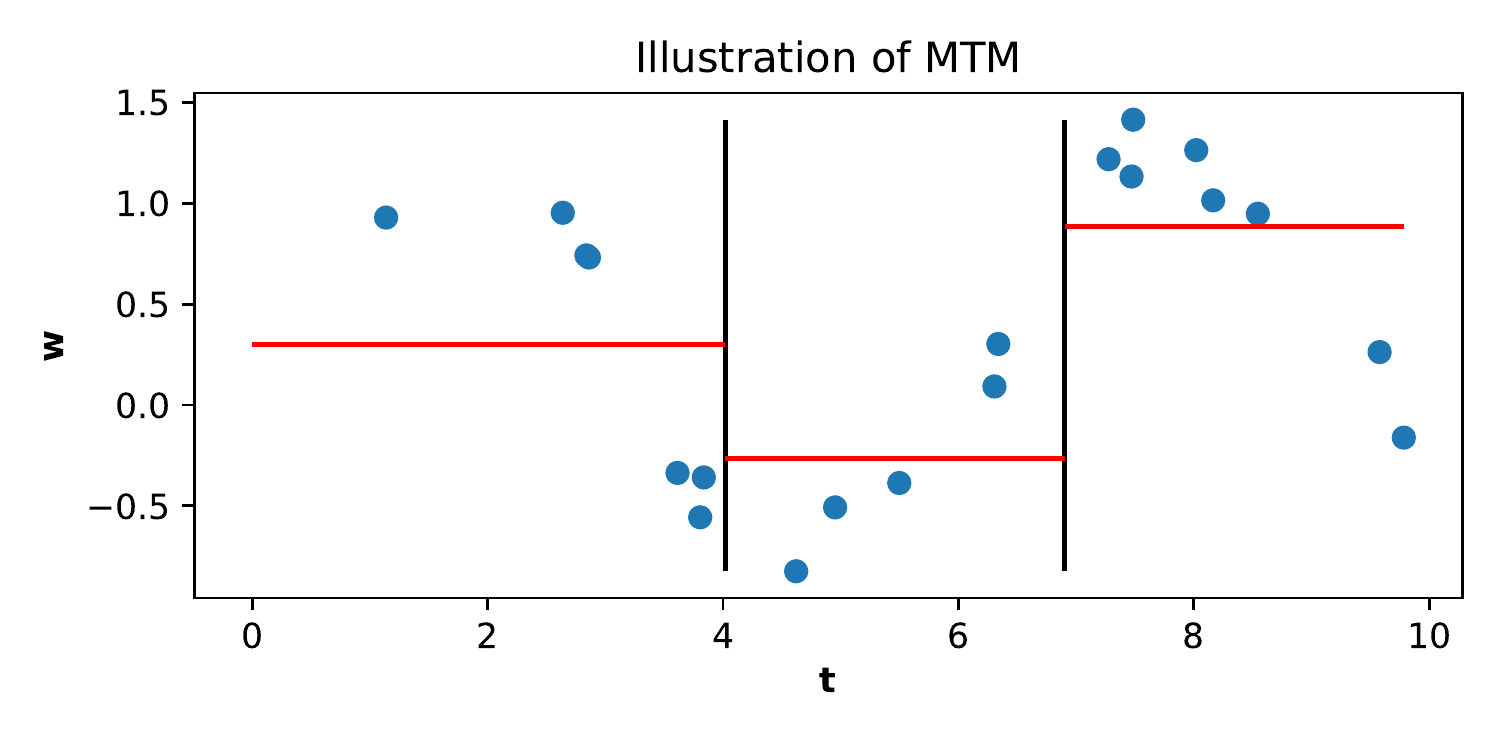}
    \end{center}
     \caption{MTM approximates the $w$ values in a particular bin by their mean (red horizontal lines) and calculates sum of squared differences of $w$ values from the corresponding means.}
     \label{figmtm}
\end{figure}
The numerator can be interpreted as a PWC ordinary least squares regression. In principle, any regression technique could be used to approximate MTM. The benefits of the PWC regression are that it has an extremely low number of parameters and can be computed efficiently.
To simplify notations, we substitute (\ref{hatbeta0}) into (\ref{mtmlin}) and introduce the formalism
\begin{equation}
\label{simplified}
D(\mathbf{t}, \mathbf{w}) = \dfrac{\Vert A\mathbf{w} - \mathbf{w}\Vert^2}{d \text{var}(\mathbf{w})},
\end{equation}
where $A=S(S^TS)^{-1}S^T$
is the projection matrix into the subspace generated by the columns of $S$. To reduce clutter, we omitted the argument $\mathbf{t}$ of $S$ and $A$, but we highlight, that both $S$ and $A$ are implied by the structure of $\mathbf{t}$.

The numerator being a least squares regression implies that $A$ is the \emph{hat-matrix}: it is idempotent, symmetric, and an orthogonal projection, thus, self-adjoint \cite{regression0}.

In principle, the PWC regression could be replaced by any regression technique. The benefits of the PWC regression to approximate MTM are that it has an extremely low number of parameters and can be computed efficiently.

An insight into the operation of the measure can be gained by recognizing some further special properties of matrix $A$ originating from its special construction from the orthogonal matrix $S$, which we utilize in Section \ref{secopt}. 

\begin{lemma} (Properties of the matrix $A$)
\label{lem-A}
Let $A$ denote a matrix $A=S(S^TS)^{-1}S^T$, and $\mathcal{I}_k\subset \lbrace 1, \dots, d\rbrace$ denote the set of indices of $\mathbf{t}$ falling in bin $k$. The matrix $A$ is a square matrix of type $\mathbb{R}^{d\times d}$ and $A_{ij}=1/\vert\mathcal{I}_k\vert$ if $i,j\in\mathcal{I}_k$, and $A_{ij}=0$ otherwise. As a consequence, $(A\mathbf{w})_i$ is the mean of elements of $\mathbf{w}$ falling in the bin where $\mathbf{t}_i$ falls.

\end{lemma}
\begin{proof}
For the proof see \ref{prooflem1}.
\end{proof}

The operation of MTM with PWC approximation is illustrated in Figure \ref{figmtm}. The paired samples of $\mathbf{t}$ and $\mathbf{w}$ are visualized in a scatter plot and 3 bins are indicated by vertical lines. By Lemma \ref{lem-A} and the simplified form (\ref{simplified}), MTM approximates the $\mathbf{w}$ values in a particular bin by their mean (red horizontal lines) and the numerator of the measure calculates the squared differences of $\mathbf{w}$ values from the corresponding the means. Thus, the numerator is the sum of residuals in the PWC regression, which is divided by the total empirical variance of $\mathbf{w}$, hence, $1 - D(\mathbf{t}, \mathbf{w})$ is the $r^2$ score of regressing $\mathbf{w}$ as the target variable on $\mathbf{t}$ as the explanatory variable using a piecewise constant regression. Another interpretation that will be discussed in detail in the next subsection is that MTM measures the uncertainty of $\mathbf{w}$ values falling in a particular bin by computing their variance (equal to the sum of residuals when $\mathbf{w}$ is approximated by the means within the bins). This uncertainty characterizes how much can $\mathbf{w}$ be treated as a function of $\mathbf{t}$.

To illustrate the structure of $A$ for a better insight and also validate the lemma qualitatively, let $\mathbf{t}= [2, 0, 5]^T$, $\mathbf{w}= [8, 2, 2]^T$ being binned to 2 bins by the binning vector $\mathbf{q}= [0, 3, 6]^T$. Then,
\begin{equation}
S= \left[\begin{array}{cc} 1 & 0 \\ 1 & 0 \\ 0 & 1\end{array}\right], \quad
A= S(S^TS)^{-1}S^T = \left[\begin{array}{ccc} 1/2 & 1/2 & 0 \\ 1/2 & 1/2 & 0 \\ 0 & 0 & 1\end{array}\right], \quad
A\mathbf{w} = \left[\begin{array}{c}5\\ 5\\ 2\end{array}\right], \nonumber
\end{equation}
where the first two elements $A\mathbf{w}$ are the means of the first two elements of $\mathbf{w}$ belonging to the first slice of the slicing of $\mathbf{t}$. 

We mention that the piecewise constant approximation enables the regularization of the measure through the number of bins. One can readily see that, if all values of the template $\mathbf{t}$ are unique, and each value is treated as a separate slice, both $S$ and $A$ become the identity matrix, and the numerator of PWC MTM $\Vert A\mathbf{w} - \mathbf{w}\Vert^2$ will give a perfect match ($0$ dissimilarity) for any $\mathbf{w}$. Consequently, the use of a low number of bins regularizes PWC MTM to prevent overfitting by controlling the smoothness of the non-linear tone mappings used to approximate $\mathbf{w}$ from $\mathbf{t}$.

Finally, another important property of $A$ is pointed out: its relation to k-means clustering \cite{bishop}.

\begin{lemma} (On the relation of the projection matrix $A$ to k-means clustering)
\label{lem-clustering}
Let $\mathcal{A}(\mathbf{t}, b)$ denote the set of all hat-matrices $A$ implied by slice matrices $S$ binning $\mathbf{t}$ to $b$ non-empty bins. Minimizing the expression
\begin{equation}
\hat{A} = \text{argmin}_{A \in \mathcal{A}(\mathbf{t}, b)} \Vert A\mathbf{t} - \mathbf{t}\Vert^2
\end{equation}
is equivalent to solving the k-means clustering problem for the elements of $\mathbf{t}$, and constructing the projection matrix $\hat{A}$ from the clusters interpreted as bins.
\end{lemma}
\begin{proof}
The expression $\Vert A\mathbf{t} - \mathbf{t}\Vert^2$ measures the sum of squared residuals when each element of $\mathbf{t}$ within a slice is approximated by the mean of the elements in the slice. Treating the slices as clusters and summing the squared residuals for each cluster, one can recognize $\Vert A\mathbf{t} - \mathbf{t}\Vert^2$ as the objective function of k-means clustering to be minimized in the space of all $b$-partitioning of the coordinates of $\mathbf{t}$.
\end{proof}

Finally, we mention that solving the k-means clustering problem is not equivalent to applying the well-known ML-EM k-means clustering algorithm, which provides only a suboptimal solution \cite{bishop}.

\subsection{The importance of binning and problem formulation}

Bin selection is the problem of determining the proper number and widths of bins \cite{binning} to group data, and the ideal binning strategy is usually data and application-specific.  

The authors of MTM \cite{mtm} did not address the question of bin selection for the slice transform and used equal width binning (EQW) in the evaluation of the measure with relatively low numbers of bins showing the best performance. Although there are numerous rules of thumb to determine the number of bins (square root rule \cite{smbinning}, Struges-rule \cite{smbinning}, Rice-rule \cite{smbinning}), as pointed out in the previous subsection, the number of bins plays the role of the regularization parameter, thus, we keep it as a degree of freedom.

On the other hand, given a particular number of bins, the selection of bin boundaries can naturally be expected to affect the performance of PWC MTM. The goal of this paper is to examine the effect of bin boundary selection strategies on PWC MTM to identify ideal binning techniques under specific conditions. We mention, that there are results in the literature for the selection of the widths of bins (Scott's rule \cite{scott}, Freedman-Diaconis' choice \cite{freedman}), and variable-width bins (like equal frequency binning (EQF) \cite{eqf} with each bin containing the same number of elements) have also been proposed. The common point of these techniques is that most of them are derived to optimize the construction of the empirical distribution function through histograms in terms of some optimality criteria. 
As MTM is intended to be used in pattern recognition scenarios to recognize patterns under some assumptions on the nature of the noise and possible distortions, them problem of binning is essentially different from that of constructing the empirical distribution function. To support this claim, we anticipate some results discussed in subsection \ref{low-eqf}, where EQF turns out to have extremely low performance in the pattern recognition settings of the numerical experiments.



\subsection{The relevance of MTM, its relation to MI, and the introduction of the term \emph{normalized Unexplained Variance}}
\label{sec-relation-to-mi}

As there are dozens of (dis)similarity measures proposed in the literature, we found it crucial to point out some beneficial properties of MTM that make it worth studying its statistical properties and motivated the writing of this paper. The authors of \cite{mtm} already showed that PWC MTM can be computed efficiently for a template and all windows of a signal, and also demonstrated that the performance of MTM in pattern recognition applications is highly competitive with that of MI. To further emphasize its relevance as a general-purpose dissimilarity measure, in this subsection we show that MTM is a highly analogous concept to a normalized variant of MI, leading us to change the nomenclature by introducing the term \emph{normalized Unexplained Variance} (nUV) to make the name more aligned to the classical concepts being utilized under the hood.


In statistics, entropy and variance are probably the most widely used measures of uncertainty \cite{uncertainty}. Related to entropy, mutual information (MI) treats the coordinates of the window $\mathbf{w}$ and template $\mathbf{t}$ as corresponding realizations of two random variables ($w$, $t$) with a joint distribution $(\mathbf{w}_i, \mathbf{t}_i) \sim p_{w, t}(w, t)$, and is defined as $I(t,w)= H(w) - H(w|t)$, where $H(w)$ and $H(w|t)$ denote the marginal and conditional entropies of the distribution $p_{w,t}(w,t)$. For (dis)similarity measures it is usually desired to map into a bounded range in order to be easily interpretable and comparable across different problems, and consequently, there are multiple normalized variants of MI proposed. One particular normalization leads to the \emph{uncertainty coefficient} \cite{numrec}, but for clearity, in rest of the paper we refer it as \emph{normalized mutual information} (nMI):

\begin{equation}
\label{nmi}
nMI(t,w) = \dfrac{I(t, w)}{H(w)}= 1 - \dfrac{H(w|t)}{H(w)} \simeq nMI(\mathbf{t},\mathbf{w}),
\end{equation}
where $nMI(\mathbf{t}, \mathbf{w})$ refers to empirical estimation using the paired samples of $\mathbf{t}$ and $\mathbf{w}$.
nMI quantifies the relative amount of information $w$ and $t$ share, by subtracting the relative amount of uncertainty remaining in $w$ given $t$ from the total amount of uncertainty in $w$, which is scaled to $1$. 

According to the law of total variance \cite{totalvar}, if $w$ has finite variance, 
\begin{equation}
\mathbb{D}^2(w) = \mathbb{E}[\mathbb{D}^2(w|t)] + \mathbb{D}^2[\mathbb{E}(w|t)]
\end{equation}
holds. The first term on the right-hand side is usually referred as the \emph{unexplained variance}, characterizing the variance remaining in $w$ given $t$; and the second term is called the \emph{explained variance} characterizing the variation of $w$ one can explain by the variation of $t$.
Rearranging the equation and dividing both sides by $\mathbb{D}^2(w)$, one obtains
\begin{equation}
\label{aaa}
 \frac{\mathbb{D}^2[\mathbb{E}(w|t)]}{\mathbb{D}^2(w)} = 1-\frac{\mathbb{E}[\mathbb{D}^2(w|t)]}{\mathbb{D}^2(w)}.
\end{equation}
Given a paired sample for $w$ and $t$, the various terms can be approximated by empirical quantities. In the theory of regression, the conditional expectation function $\mathbb{E}(w|t)$ is proved \cite{bishop} to be the best predictor of $w$ from $t$ in terms of squared loss, and
any least squares regression function estimating $w$ from $t$ can be treated as an approximation of $\mathbb{E}(w|t)$. Consequently, $\mathbb{E}[\mathbb{D}^2(w|t)]$ can be estimated by the squared residuals of the regression. Choosing a piecewise constant regression function and using the notations introduced before,
\begin{equation}
    \mathbb{E}[\mathbb{D}^2(w|t)] \simeq \frac{\Vert A\mathbf{w} - \mathbf{w}\Vert^2}{d},
\end{equation}
with $d$ denoting the dimensionality of the feature space, and estimating $\mathbb{D}^2(w)$ by the empirical variance $\text{var}(\mathbf{w})$:
\begin{equation}
\label{bbb}
\frac{\mathbb{D}^2[\mathbb{E}(w|t)]}{\mathbb{D}^2(w)} = 1 - \frac{\mathbb{E}[\mathbb{D}^2(w|t)]}{\mathbb{D}^2(w)} \simeq 1 - D(\mathbf{t}, \mathbf{w}).
\end{equation}
The term on the left hand side is the normalized explained variance, quantifying the variance of $\mathbf{w}$ explained by $\mathbf{t}$. 
Comparing the expressions (\ref{nmi}) and (\ref{bbb}), one can observe that nMI and $1 - D(\mathbf{t}, \mathbf{w})$ are highly analogous concepts, both quantifying the uncertainty \emph{disappearing} from $\mathbf{w}$ given $\mathbf{t}$, nMI using entropy, and $1 - D(\mathbf{t}, \mathbf{w})$ using variance to measure uncertainty. 


This analogy shows that MTM is more general in principles than what the name \emph{Matching by Tone Mapping} would suggest, it can be treated as a meaningful alternative of MI in any applications where MI is used as a similarity measure. To emphasize its generality beyond image processing, be conformant with the literature of statistics, and the classical principles utilized by its operation, we found it necessary to change the nomenclature and this change is formalized in the next definition.

\begin{definition}
In the context of this paper, the piecewise constant approximation of the \emph{normalized unexplained variance} (nUV) dissimilarity measure refers to the piecewise constant approximation of the MTM measure, both denoted by $D(\mathbf{t}, \mathbf{w})$.
\end{definition}

As it is a common technique to estimate nMI through binning \cite{mi-estimation}, we compare the role of binning in nMI to that in PWC nUV. Unlike the binning implementations of nMI (discretizing both vectors), PWC nUV carries out binning only for the template vector. Consequently, another beneficial property of PWC nUV is that one can expect less loss of information and less sensitivity to improper binning strategies.


Finally, we mention that although MI and its variants are widely used in pattern recognition, according to our best knowledge, there are no results in the literature to optimize its operation by determining the ideal binning technique in terms of some pattern recognition specific optimality conditions, which could make PWC nUV a favorable choice to nMI in problems where the assumptions of optimality are met.


\section{The optimal binning technique}
\label{secopt}
In this section, we carry out the statistical analysis of bin boundary selection strategies for the PWC nUV measure. 
First, we introduce a statistical model and phrase the optimality criterion we aim to optimize by choosing the binning technique. Then, the main results on the optimality of binning are derived and two algorithms are proposed to determine the ideal binning for a particular template $\mathbf{t}$ under mild assumptions on the nature of distortions and noise. 

\subsection{The optimal operation of the PWC nUV dissimilarity measure}

Naturally, any binning technique (equal width, equal frequency, etc.) can be used to carry out the slice transform, thus, to drive the PWC nUV measure. In order to select the \emph{optimal} binning technique for a given number of bins, we need to define when we consider the operation of the measure \emph{optimal}. As the goal of pattern recognition is to recognize patterns under certain classes of distortions, we consider PWC nUV operating optimally when it separates the noisy background from a noisy and distorted pattern as much as possible. In order to put this concept formally -- in accordance with the notations so far -- let $\mathbf{t}$ denote a template, $\mathbf{\boldsymbol \xi}\sim \prod\limits_{i=1}^{d}f(\mathbf{\boldsymbol \xi}_i)$ a window containing white noise (from distribution $f$ with finite variance $\sigma^2$), and $\mathbf{w}= \mathcal{M}[\mathbf{t}] + \mathbf{\boldsymbol \zeta}$, $\mathbf{\boldsymbol \zeta} \sim \prod\limits_{i=1}^{d}f(\mathbf{\zeta}_i)$ the window containing the template distorted by the tone mapping $\mathcal{M}$ and additive noise $\mathbf{\boldsymbol \zeta}$. We consider the tone mapping to be a stochastic process (a random real function) $\lbrace \mathcal{M}(x) \rbrace_{x \in\mathbb{R}}$ with finite first and second moments ($\mathbb{E}[\mathcal{M}(x)] < \infty$, $\mathbb{D}^2[\mathcal{M}(x)] < \infty$, $x\in\mathbb{R}$) defined by all finite dimensional probability distributions $g_{x_1, x_2, \dots, x_n}^n(x_1, x_2, \dots, x_n)$, $n <\infty$.
\begin{definition} (Optimal binning)
\label{def-optimality}
With the notations introduced before, for a given template $\mathbf{t}$ we consider the operation of PWC nUV optimal if its expected discrimination power 
\begin{equation}
\label{criterion_0}
\mathbb{E}_\xi \mathbb{E}_{\zeta}\mathbb{E}_{\mathcal{M}} [ D(\mathbf{t}, \mathbf{\xi}) -  D(\mathbf{t}, \mathbf{w})] = \mathbb{E}_\xi  D(\mathbf{t}, \mathbf{\xi}) -  \mathbb{E}_{\zeta}\mathbb{E}_{\mathcal{M}} D(\mathbf{t}, \mathbf{w})
\end{equation}
regarding a noisy window and a noisy distorted template is maximal. 
\end{definition}
Put in another way, we consider PWC nUV operating optimally if the binning technique used in the slice transform is such that the expected dissimilarity of the template from noise ($D(\mathbf{t}, \mathbf{\xi})$) and the expected dissimilarity of the template from the distorted template ($D(\mathbf{t}, \mathcal{M}[\mathbf{t}] + \mathbf{\zeta})$) is as different as possible.


\subsection{Linearization of the distortion and the models being examined}\label{section32}

In this subsection, we apply equivalent transformations to replace the stochastic process $\mathcal{M}$ by a random variable from a finite-dimensional distribution.

If all the coordinates of $\mathbf{t}$ are different, $\mathcal{M}[\mathbf{t}]$ is a random quantity governed by the $d$-dimensional distribution $g^d_{\mathbf{t}_1, \mathbf{t}_2, \dots, \mathbf{t}_d}$ of the distortion $\lbrace \mathcal{M}(x)\rbrace_{x\in\mathbb{R}}$. What makes $\mathcal{M}[\mathbf{t}]$ different from a real random vector variable is the presence of equal coordinates. If $\mathbf{t}_i = \mathbf{t}_j$ holds for some $i\neq j$, $\mathcal{M}[\mathbf{t}]_i = \mathcal{M}[\mathbf{t}]_j$ also needs to hold, as one realization of the random tone mapping is a function assigning the same domain value $\mathcal{M}(\mathbf{t}_i) = \mathcal{M}(\mathbf{t}_j)$ to both $\mathbf{t}_i$ and $\mathbf{t}_j$. In order to ensure that this condition holds, we factor $\mathbf{t}$ by introducing its full-rank slice transform. Let $d_{\boldsymbol\tau}$ denote the number of unique coordinates of $\mathbf{t}$, $\boldsymbol\tau\in\mathbb{R}^{d_\tau}$ denote the vector of unique elements of $\mathbf{t}$ in an increasing order, and let $S_{\boldsymbol\tau}\in\lbrace 0, 1\rbrace^{d\times n_{\boldsymbol\tau}}$ denote the \emph{full-rank} slice transform matrix, in which each unique coordinate of $\mathbf{t}$ falls in a distinct slice and let $\mathbf{n}_{\boldsymbol\tau} \in\mathbb{Z}_+^{d_{\boldsymbol\tau}}$ denote the vector of the number of unique elements in the slices of $S_{\boldsymbol\tau}$. With $S_{\boldsymbol\tau}$, $\mathbf{t}$ can be reconstructed without any loss of information, that is $\mathbf{t}= S_{\boldsymbol\tau} \mathbf{\boldsymbol\tau}$. 

With these notations, $\mathcal{M}[\mathbf{t}] = S_{\boldsymbol\tau}\mathcal{M}[\mathbf{\boldsymbol\tau}]$, where $\mathcal{M}[\mathbf{\boldsymbol\tau}]$ is a random vector from the $d_\tau$ dimensional distribution $g^{d_{\boldsymbol\tau}}_{{\boldsymbol\tau}_1, {\boldsymbol\tau}_2, \dots, {\boldsymbol\tau}_{d_{\boldsymbol\tau}}}$. Consequently, $\mathbf{w}= S_{\boldsymbol\tau}\mathbf{m} + \mathbf{\boldsymbol\zeta}$, where $\mathbf{m} = \mathcal{M}[\mathbf{\boldsymbol\tau}] \sim g^{d_{\boldsymbol\tau}}_{{\boldsymbol\tau}_1, {\boldsymbol\tau}_2, \dots, {\boldsymbol\tau}_{d_{\boldsymbol\tau}}}$ and the optimality criterion becomes the maximization of 
\begin{equation}
\label{optimality-criterion}
\mathbb{E}_{\boldsymbol\xi}D(\mathbf{t}, \boldsymbol\xi) - \mathbb{E}_{\boldsymbol\zeta}\mathbb{E}_\mathbf{m} D(\mathbf{t}, S_{\boldsymbol\tau}\mathbf{m} + \boldsymbol\zeta).
\end{equation}

\subsection{The need for first order approximation of the expected values}
\label{ratio-dist}
Treating the $D(\mathbf{t}, \mathbf{w})$ as a random quantity through the random nature of the window implies some difficulties in the evaluation of the optimality-criterion (\ref{optimality-criterion}), as random variables appear in both the numerator and denominator, leading to a ratio distribution which is analytically interactable. In order to carry out the analysis, we introduce the usual first-order approximation for the expectated value of the ratio distribution \cite{ratio-distribution}:
\begin{equation}
\mathbb{E}_\eta\dfrac{f(\eta)}{g(\eta)} \simeq \dfrac{\mathbb{E}_\eta f(\eta)}{\mathbb{E}_\eta g(\eta)},
\end{equation}
and use this approximation throughout the paper when evaluating the expected value of the PWC nUV measure. 



\subsection{Statistical analysis of the model}


In this section we carry out the statistical analysis of the model introduced.

\begin{proposition}  
\label{st-noise}
Using the notations introduced before, in the first order approximation of the ratio distribution,
\begin{equation}
\mathbb{E}_{\boldsymbol\xi} D(\mathbf{t}, \boldsymbol\xi)\simeq\dfrac{d-b}{d-1}.
\end{equation}
\end{proposition}
\begin{proof}
For the proof see \ref{proof-noise}.
\end{proof}

As a consequence of the proposition, the expected value of the dissimilarity of a template $\mathbf{t}$ from a window containing only noise is a constant, which depends only on the dimensionality of the space $d$ and the number of bins $b$, but is independent of the structure of the bins (matrix $A$). This result suggests that the matrix $A$ minimizing the term $\mathbb{E}_\mathbf{m}\mathbb{E}_{\boldsymbol\zeta} D(\mathbf{t}, S_\tau\mathbf{m} + \boldsymbol\zeta)$ will maximize the expected discrimination power of the measure according to the the optimality criterion (\ref{optimality-criterion}).

\begin{proposition} (On the expected dissimilarity of a template and a distorted, noisy template)
\label{st-distorted}
With the notations introduced before,
\begin{equation}
\label{the-solution}
\mathbb{E}_\zeta\mathbb{E}_\mathbf{m} D(\mathbf{t}, S_{\boldsymbol\tau}\mathbf{m} + \boldsymbol\zeta) \simeq \\ 
\dfrac
{
\langle \mathbf{n}_{\boldsymbol\tau}, \mathbb{E}_\mathbf{m}\mathbf{m}^2\rangle - \langle A, S_{\boldsymbol\tau} \text{Cross}(\mathbf{m}) S_{\boldsymbol\tau}^T\rangle_{F} + \sigma^2(d - b)
}
{\langle \mathbf{n}_{\boldsymbol\tau}, \mathbb{E}_\mathbf{m}\mathbf{m}^2\rangle - \dfrac{1}{d}\mathbf{n}_{\boldsymbol\tau} \text{Cross}(\mathbf{m}) \mathbf{n}_{\boldsymbol\tau}^T + \sigma^2\left(d - 1\right)},
\end{equation}
where $\text{Cross}(\mathbf{m})_{ij}= \mathbb{E}_\mathbf{m}\mathbf{m}_i\mathbf{m}_j$ denotes the expected cross-product matrix of the distortion $\mathbf{m}$, $\langle, \rangle_F$ denotes the Frobenius inner product, and the vector $\mathbf{n}_{\boldsymbol\tau}$ contains the cardinalities of the slices in the full-rank slice transform $S_{\boldsymbol\tau}$ of $\mathbf{t}$.
\end{proposition}

\begin{proof}
For the proof see \ref{proof-distorted}.
\end{proof}

We note that the cross-product matrix ($Cross(\mathbf{m})$) encodes the covariance structure and the mutual relationships of the coordinates of the mean vector as
$
Cross(\mathbf{m})= Cov(\mathbf{m}) + (\mathbb{E}_{\mathbf{m}}\mathbf{m})(\mathbb{E}_{\mathbf{m}}\mathbf{m})^T
$
where $Cov(\mathbf{m})_{ij}=\mathbb{E}_{\mathbf{m}}[(\mathbf{m}_i - \mathbb{E}_\mathbf{m}\mathbf{m}_i)(\mathbf{m}_j - \mathbb{E}_\mathbf{m}\mathbf{m}_j)]$ is the exact covariance matrix of the distortion.

\begin{theorem} (On the optimization of binning)
\label{the-theorem}
In first order approximation of the expected value of the ratio distribution, the projection matrix maximizing the expression
\begin{equation}
\label{maximization}
\hat{A} = \text{argmax}_{A \in \mathcal{A}(\mathbf{t}, b)} \langle A, S_{\boldsymbol\tau}\text{Cross}(\mathbf{m})S_{\boldsymbol\tau}^T\rangle_F 
\end{equation}
minimizes the expected dissimilarity of the template and the distorted, noisy template (equation (\ref{the-solution})), thus, maximizes the separation power of the measure for the distortion $\mathbf{m}$. 
\end{theorem}
\begin{proof}
The statement can be readily seen as the projection matrix $A$ appears only in the numerator of (\ref{the-solution}) and has a negative sign, thus maximizing (\ref{maximization}) minimizes the approximation in equation (\ref{the-solution}) and maximizes the optimality criterion in equation (\ref{optimality-criterion}).
\end{proof}

The results give an interesting insight into the operation of the PWC nUV measure. In order to optimize its operation in terms of the optimality criterion (\ref{optimality-criterion}), one needs to find a bin structure implying a projection matrix $A$, which maximizes the alignment of the projection matrix and the cross-product matrix of the distortion by maximizing their Frobenius (matrix) inner product. Put in another way, as $A_{ij}$ is zero if $\mathbf{t}_i$ and $\mathbf{t}_j$ do not fall in the same bin, the maximization of the Frobenius-product requires strongly covarying distortion coordinates having similar means to fall in the same bin to contribute their high cross-product value to the objective function of the maximization (\ref{maximization}).

One can readily see that the optimization problem (\ref{maximization}) is a combinatorial optimization problem as the matrices in the set $\mathcal{A}(\mathbf{t}, b)$ are induced by the $b$-partitions of the set $\left\lbrace 1, \dots, d \right\rbrace$, thus, the problem is hardly tractable analytically. However, greedy algorithms \cite{greedy} can be derived to approximate the optimal solutions. Given a cross-product matrix $Cross(\mathbf{m})$, the number of bins $b$ and the vector $\mathbf{n}_{\boldsymbol\tau}$ containing the cardinalities of the slices of the full-rank slice transform $S_{\boldsymbol\tau}$, a greedy algorithm to find a binning approximating the ideal one is provided in Algorithm \ref{alg1}. The algorithm initializes a random configuration of $b$ non-empty bins and computes the inner product $\langle A, S_{\boldsymbol\tau} Cross(\mathbf{m})S_{\boldsymbol\tau}^T\rangle_F$ with the matrix $A$ implied by the random configuration. Then, iteratively checks if moving any of the bin boundaries one step to the left or right increases the inner product. In each iteration, the adjustment of bin boundaries leading to the highest increase in the inner product is being chosen. The algorithm stops when no further increase can be achieved. The vector of bin boundaries $\mathbf{q}$ computed by the algorithm contains the indices of the bin boundaries of the ideal binning in the vector $\boldsymbol\tau$ containing the ordered, unique elements of $\mathbf{t}$.

{
\captionof{algorithm}{Greedy optimization of binning. Accessing an item in a vector is denoted by squared brackets, subscripts and superscripts are parts of the names of the variables. A Python implementation of the algorithm is available in the GitHub repository \url{https://github.com/gykovacs/ideal_binning_nuv}}\label{alg1}
\begin{footnotesize}

//\Call{RowColSum}{} computes the contribution of row and column $i$ to the inner product for bin $j$
\begin{algorithmic}[1]
\Function{RowColSum}{$i$, $j$, $q$, $Cross$, $n_\tau$}
\State $s\gets Cross[i, i]*n_\tau[i]*n_\tau[i]$
\For{$k \gets q[j]$ to $q[j+1]$} // the upper bound exclusive
    \If{$k \neq i$}
        \State $s \leftarrow s + 2*Cross[i, k]*n_\tau[i]*n_\tau[k]$
    \EndIf
\EndFor
\State\Return s;
\EndFunction
\end{algorithmic}

//\Call{Change}{} computes the change in the objective function when the $i$th bin boundary is moved one step to the left ($\delta = -1$) or to the right ($\delta = 1$)
\begin{algorithmic}[1]
\Function{Change}{$i,\delta, s, q, n, Cross, n_\tau$}
\State $\theta \leftarrow (\delta-1)/2$ // transforming the step to 0 or -1
\State $s'\leftarrow s_i - \delta*$\Call{RowColSum}{$q[i] + \theta, i, q, Cross, n_\tau$} // new contribution of bin $i$
\State $s_{-1}'\leftarrow s_{i-1} + \delta*$\Call{RowColSum}{$q[i] + \theta, i-1, q, Cov, n_\tau$} // new contribution of bin $i-1$
\State $n', n_{-1}' \leftarrow n[i] - \delta*n_\tau[q[i] + \theta], n[i-1] + \delta*n_\tau[q[i] + \theta]$ // new cardinalities
\State $\Delta, \Delta_{-1} \leftarrow s'/n' - s[i]/n[i],  s_{-1}'/n_{-1}' - s[i-1]/n[i-1]$ // changes in the objective function
\State \Return $(\Delta + \Delta_{-1}, s', s_{-1}', n', n_{-1}')$
\EndFunction
\end{algorithmic}
// \Call{GreedyBinning}{} implements the proposed greedy binning algorithm
\begin{algorithmic}[1]
\Function{GreedyBinning}{$Cross$, $n_\tau$, $b$}
\State// The parameters $Cross$, $n_\tau$ and $b$ denote the cross-product matrix to fit, the vector $\mathbf{n}_\tau$ and the number of bins, respectively.
\State $q \leftarrow$ A random vector of size $(b+1)$ containing increasing integers from $0$ to $d$, with q[0]=0$, q[b+1]=d$ 
\State $s, n \leftarrow$ Vectors of size $b$, initialized by zeros.
\State $T \leftarrow 0$ (The inner product (target function) to be maximized)
\State //Initializing the target function and the vectors containing the sums and number of items related to the bins
\For{$i \gets 0$ to $b$}
\For{$j \gets q[i]$ to $q[i+1]$}
\For{$k \gets q[i]$ to $q[i+1]$}
\State $s[i]\leftarrow s[i] + Cross[j,k]$ // the initial contributions of the bins
\EndFor
\State $n[i]\leftarrow n[i] + \mathbf{n}_\tau[i]$ // the initial cardinalities of the bins
\EndFor
\State $T \leftarrow T + s[i]/n[i]$ // the initial objective function
\EndFor
\State // Iteratively checking for the largest improvement by moving bin boundaries one step to the left or right
\Do 
\State $\Delta^* \leftarrow 0$
\For{$i\gets 0$ to $b$} // for all bins (the upper bound exclusive)
\For{$\theta \in \lbrace -1, 0\rbrace$} // relative index of the shrinking bin
\If{$n[i+\theta] > \mathbf{n}_u[q[i]+\theta]$} // if the shrinking bin does not get empty
\State $\delta' \leftarrow 2*\theta + 1$ // transforming relative index to a step -1/+1
\State $\Delta', s', s_{-1}', n', n_{-1}' \leftarrow $\Call{Change}{$i, 2*\theta + 1, s, q, n, Cross, n_\tau$} // the changes
\If{$\Delta' > \Delta^*$} // if the improvement larger than before, record its parameters
\State $\Delta^*, i^*, \delta^*, s^*, s_{-1}^*, n^*, n_{-1}^* \leftarrow \Delta', \delta', i, s', s_{-1}', n', n_{-1}'$
\EndIf
\EndIf
\EndFor
\EndFor
\If{$\Delta^* > 0$} // if there was an improvement, update the values accordingly
\State $T, q[i^*] \leftarrow T + \Delta^*, q[i^*] + \delta$
\State $s[i^*], s[i^*-1], n[i^*], n[i^*-1]\leftarrow s^*, s_{-1}^*, n^*, n_{-1}^*$
\EndIf
\doWhile{$\Delta > 0$}
\State \Return $q$
\EndFunction
\end{algorithmic}
\end{footnotesize}
}
As the following corollary shows, if the distortion is spherical and centered to the origin, that is, it makes no distinction between various directions of the feature space, the expected value reduces to a constant -- in accordance with the \emph{no free lunch theorems} of machine learning \cite{nfl0}: all machine learning techniques of a class (in this case all binnings of $b$ bins) provide the same average performance when evaluated on all possible problems with no structural preference in their distribution (in this case all distorted vectors from some spherical distribution).
\begin{corollary} (No free lunch theorem)
If all elements of $\mathbf{t}$ are unique, and the the distortion has a spherical distribution in the d-dimensional features space of $\mathbf{t}$, that is, $Cov(\mathbf{m})= \mathbb{I}\sigma_\mathbf{m}^2$,
\begin{equation}
\mathbb{E}_\zeta\mathbb{E}_\mathbf{m}D(\mathbf{t}, S_{\boldsymbol\tau}\mathbf{m} + \zeta) \simeq 
\dfrac{\sigma^2_\mathbf{m}(d - b) + \sigma^2(d - b)}{\sigma^2_\mathbf{m}(d - 1) + \sigma^2(d-1)}.
\end{equation}
\end{corollary}
\begin{proof}
One can readily see by substituting $\sigma_\mathbf{m}^2$ in place of $\mathbb{E}_\mathbf{m}\mathbf{m}^2$ and $\mathbb{I}\sigma_m^2$ in place of $Cross(\mathbf{m})$ in Proposition \ref{st-distorted}, and utilizing $tr(A)=b$. Without the unicity constraint on the elements of $\mathbf{t}$, the ties imply non-diagonal non-zero entries in $S_{\boldsymbol\tau} Cross(\mathbf{m})S_{\boldsymbol\tau}^T$, and $\langle A, S_{\boldsymbol\tau} Cross(\mathbf{m})S_{\boldsymbol\tau}^T\rangle_F$ would not reduce to $tr(A)\sigma^2_\mathbf{m}=b\sigma^2_\mathbf{m}$.
\end{proof}

As a special case of Proposition \ref{st-distorted}, one can expect that the distortion $\mathbf{m}$ is such that it maps $\mathbf{t}$ \emph{close} to itself. This closeness can be modelled by a distribution which has the mean $\mathbb{E}_\mathbf{m}\mathbf{m} = \boldsymbol\tau$, where $\boldsymbol\tau$ is the vector of unique elements of $\mathbf{t}$ in an increasing order. The following proposition provides an insight into the effect of localized distortions: in this case, the ideal quantization requires $A$ to match the covariance structure of the distortion and also requires the minimization of the representation error of $\mathbf{t}$ made by the binning.

\begin{proposition} (On the expected value of the measure with a localized distortion)
\label{st-reduced-1}
Using the notations introduced before, with $\mathbf{m'} = \mathbf{m} - \boldsymbol\tau$ if $\mathbb{E}_\mathbf{m}\mathbf{m} = \boldsymbol\tau$, then
\begin{equation}
\label{reduced-1}
\mathbb{E}_{\boldsymbol\zeta}\mathbb{E}_\mathbf{m} D(\mathbf{t},S_{\boldsymbol\tau}\mathbf{m} + \boldsymbol\zeta)\simeq\\
\dfrac{\Vert A\mathbf{t} - \mathbf{t} \Vert^2 + \langle \mathbf{n}_{\boldsymbol\tau}, \mathbb{E}_\mathbf{m}\mathbf{m'}^2\rangle - \langle A, S_{\boldsymbol\tau}\text{Cov}(\mathbf{m'})S_{\boldsymbol\tau}^T\rangle_F + \sigma^2(d - b)}
{d\text{var}(\mathbf{t}) + \langle \mathbf{n}_{\boldsymbol\tau}, \mathbb{E}_\mathbf{m}\mathbf{m'}^2\rangle - \dfrac{1}{d}\mathbf{n}_{\boldsymbol\tau} Cov(\mathbf{m'}) \mathbf{n}_{\boldsymbol\tau}^T + \sigma^2\left(d - 1\right)}.
\end{equation}

\end{proposition}
\begin{proof}
For the proof, see Appendix \ref{proof-reduced-1}
\end{proof}

As a consequence of the proposition, if the distortion is centered to $\mathbf{t}$ in the sense that $\mathbb{E}_\mathbf{m}\mathbf{m} = \boldsymbol\tau$, the ideal binning jointly minimizes the representation error of the binning $\Vert A\mathbf{t}- \mathbf{t}\Vert^2$ and the alignment of the binning with the covariance structure of the distortion $Cov(\mathbf{m})$. According to Lemma \ref{lem-clustering}, the representation error could be minimized by solving the k-means clustering problem (applying some k-means clustering technique like the well known ML-EM), however, the alignment of the binning and the covariance structure is not minimized by it, thus, in these cases still the optimization method formulated in Theorem \ref{the-theorem} and Algorithm \ref{alg1} is recommended.

It is also reasonable to suppose that the distortion is not only centered to $\mathbf{t}$, but spherical. The following proposition and theorem show that in these cases solving the k-means clustering problem leads to the ideal quantization.

\begin{proposition} (On the expected value of the measure with spherically distributed distortion)
\label{st-reduced-2}
If $\mathbf{t}$ has unique elements, $\mathbf{E}_\mathbf{m}\mathbf{m} = \boldsymbol\tau$ and $Cov(\mathbf{m} - \boldsymbol\tau) = \mathbb{I}\sigma_\mathbf{m'}^2$, then
\begin{equation}
\label{reduced-2}
\mathbb{E}_{\zeta}\mathbb{E}_\mathbf{m} D(\mathbf{t},S_{\boldsymbol\tau}\mathbf{m} + \boldsymbol\zeta)\simeq\\
\dfrac{\Vert A\mathbf{t} - \mathbf{t} \Vert^2 + \sigma_{\mathbf{m}'}^2(d - b) + \sigma^2(d - b)}
{d\text{var}(\mathbf{t}) + \sigma_{\mathbf{m'}}^2\left(d - 1\right) + \sigma^2\left(d - 1\right)}.
\end{equation}
\end{proposition}
\begin{proof}
Substituting the evaluations 
\begin{equation}
\langle \mathbf{n}_{\boldsymbol\tau}, \mathbb{E}_\mathbf{m}\mathbf{m'}^2\rangle = d\sigma^2_\mathbf{m'}, \quad\quad\\
\langle A, S_{\boldsymbol\tau} Cov(\mathbf{m'})S_{\boldsymbol\tau}^T\rangle_F = b\sigma^2_\mathbf{m'}, \quad\quad\\
\mathbf{n}_{\boldsymbol\tau} Cov(\mathbf{m'})\mathbf{n}_{\boldsymbol\tau}^T = d\sigma^2_\mathbf{m'}\nonumber
\end{equation}
into (\ref{reduced-1}) with $\mathbf{m'}=\mathbf{m} - \boldsymbol\tau$ completes the proof.
\end{proof}

\begin{theorem} (On the optimization of the binning for spherically distributed distortion)
\label{th2}
If the elements of $\mathbf{t}$ are unique and the distortion is centered to $\mathbf{t}$ with a spherical distribution, the ideal binning can be determined by solving the k-means clustering problem for the elements of $\mathbf{t}$.
\end{theorem}
\begin{proof}
The theorem is a consequence of Lemma \ref{lem-clustering} and $\Vert A\mathbf{t}- \mathbf{t}\Vert^2$ being the only $A$ dependent term in the numerator of (\ref{reduced-2}).
\end{proof}
As a consequence of Theorem \ref{th2}, when only white noise is expected, still the solution of the k-means clustering problem provides the ideal binning. 

We highlight that Theorem \ref{the-theorem} provides the general conditions of ideal binning applicable to any assumptions on $Cross(\mathbf{m})$. The greedy algorithm proposed in Algorithm \ref{alg1} finds an approximating solution but due to the combinatorial nature of the problem, it does not guarantee a global optimum. When the distortions imply that bin selection turns into the k-means clustering problem, advanced techniques developed to find the exact solution of the k-means clustering problem in 1D can be exploited to find the ideal solution \cite{exact-k-means}.
Finally, one can readily see, that although the results are based on the first-order approximation, when the unexplained variance measure is not normalized (making it analogous to MI), the same conditions on the optimal binning are exact.



\subsection{Estimation of the cross-product matrix of the distortion}

In order to determine the ideal quantization for a template, one needs to make assumptions on the cross-product structure of the expected distortions $\mathbf{m}$. 
If the distortions are known to come from a particular class of functions, one can estimate the cross-product matrix by sampling the class of functions, apply each sample function to the unique elements of $\mathbf{t}$ and compute the empirical cross-product matrix of the resulting vectors. For example, given a template $\mathbf{t}$ in an image processing application, expecting gamma distortions ($\mathcal{M}(\mathbf{t})= \mathbf{t}^\gamma$), with $\gamma\in[1, 10]$ (related to over-exposition), one can determine $\boldsymbol\tau$ as the vector of unique elements in $\mathbf{t}$, and estimate the $Cross(\mathbf{m})$ matrix by sampling $\gamma_i, i=1, \dots, N$ from $\mathcal{U}(1, 10)$ and
$
Cross(\mathbf{m}) \simeq \sum\limits_{i=1}^{N} \mathbf{m}_i \mathbf{m}_i^T/N,
$
where $\mathbf{m}_i = \boldsymbol\tau^{\gamma_i}$.

\subsubsection{The probability of ideal binning}
Working with digital signals, due to the finite precision of representation and/or sensitivity of acquisition devices, the space of windows is usually a bounded subset $\mathcal{B}\subset\mathbf{R}^d$ with volume $V(\mathcal{B})$. If the assumptions on $Cross(\mathbf{m})$ are valid in a $V^{*}$ volume of the feature space and the ideal bins are determined by these assumptions, one can readily see that the probability of the ideal operation of the nUV measure becomes $\dfrac{V^{*}}{V(\mathcal{B})}$, thus, using the proposed binning techniques can still improve the separation power of PWC nUV, even though the assumptions are valid only in a subset of the entire space $\mathcal{B}$.

\section{Tests and Results}
\label{secres}

As a dissimilarity measure highly analogous to MI, nUV has numerous potential applications from template matching through registration \cite{mi-book} to feature selection in machine learning \cite{mi-feature}. Due to this generality of the measure, the theoretical nature of the results we derived, and space limitations, we do not evaluate the measure on real data. The goal of the numerical experiments is twofold, summarized as follows.

\emph{Testing the accuracy of first-order approximations.}
By simulations we show quantitatively and illustrate qualitatively that the formulae derived in the previous section are aligned with the measurements, thus,  the first-order statistical approximation is acceptable in the scope of the experimental settings. This is carried out by simulating templates, noisy windows, and distorted templates and computing and comparing the predictions of propositions \ref{st-noise}, \ref{st-distorted} and \ref{st-reduced-2} with the real dissimilarity scores.

\emph{Testing the pattern recognition performance in terms of AUC.}
We characterize quantitatively how much improvement can be achieved by using the proposed binning techniques in pattern recognition scenarios. In each test case, for each binning technique we record if $D(\mathbf{t}, S_{\boldsymbol\tau}\mathbf{m} + \boldsymbol\zeta) < D(\mathbf{t}, \boldsymbol\xi)$ holds (indicating the correct recognition of the distorted template). We note that the percentage of correct recognitions is an estimation of the probability that a randomly chosen positive sample (distorted template) will have a smaller dissimilarity score than a randomly chosen negative sample (noisy window). This estimation is equivalent to one of the common interpretations of the widely used AUC score (Area Under the receiver operating Curve) \cite{auc}.


All results reproducable by the codes with fixed random seeds in the GitHub repository \url{https://github.com/gykovacs/ideal_binning_nuv}.

\subsection{One test case}

The computational steps in one test case are summarized as follows.

\emph{Sampling of a random template.} A random template ($\mathbf{t}\in\mathbb{R}^d$) is generated with a random dimensionality $d\in\lbrace 100,\dots,1000\rbrace$; the templates are vectors from three different distributions with equal probabilities: standard normal distribution ($\mathbf{t}\sim \mathcal{N}^d(0, 1)$), uniform distribution ($\mathbf{t}\sim \mathcal{U}^d(0, 1)$), and a distribution composed from two normals ($\mathbf{t}_i \sim \mathcal{N}(0, 1) \text{ or } \mathcal{N}(2, 1)$) -- this composition is used to generate templates with non-unimodal intensity distributions. Then, the templates are normalized into the range $[0, 1]$ and a random exponent $\gamma\in\lbrace 1/3, 1/2, 1, 2, 3 \rbrace$ is used to adjust the template by $\mathbf{t}^\gamma$ to alter its intensity distribution (in image processing this type of exponential adjustment is related to under- and over-exposition). Finally, when general distortions are evaluated with greedy binning, the intensity values in the template are rounded to 3 digits, in this way introducing some ties between template values which is also usual in digital signal processing. When Theorem \ref{th2} is examined, due to the unicity constraint on the values of $\mathbf{t}$, rounding is not applied.

\emph{Sampling of a noisy window.} A noisy window ($\boldsymbol\xi$) and the white noise vector used to distort the template ($\boldsymbol\zeta$) is generated from a normal distribution with uniformly random standard deviation $\sigma\sim\mathcal{U}(0.1, 2.0)$.

\emph{Sampling a distorted template and the cross-product structure.} 

First, the full-rank decomposition of $\mathbf{t}$ is being determined: $\mathbf{t}= S_{\boldsymbol\tau} \boldsymbol\tau$. For general distortions, a random mean vector $\mathbf{\boldsymbol\mu}_\mathbf{m}\sim\mathbf{N}^{d_{\boldsymbol\tau}}(0, 1)$ and a covariance matrix $Cov(\mathbf{m})$ are sampled. For spherical distortions, $\boldsymbol\mu_\mathbf{m} = \boldsymbol\tau$ and $Cov(\mathbf{m}) = \mathbb{I}\sigma_\mathbf{m}^2$, with $\sigma_\mathbf{m}^2 \sim \mathcal{U}(0.1, 2.0)$
    The cross-product matrix is computed as $Cross(\mathbf{m}) = Cov(\mathbf{m}) + \boldsymbol\tau \boldsymbol\tau^T$.
 A distortion vector is being sampled as $\mathbf{m}\sim\mathcal{N}(\boldsymbol\mu_{\mathbf{m}}, Cov(\mathbf{m}))$ and a distorted template ($S_{\boldsymbol\tau}\mathbf{m} + \boldsymbol\zeta$) is generated.

\emph{Binning.} The binning of the template $\mathbf{t}$ is carried out by equal width (EQW), equal frequency (EQF), k-means and greedy binning for various bin numbers. 

\emph{Calculation of dissimilarity scores.} The approximations of the expected values by Propositions \ref{st-noise}, \ref{st-distorted} and \ref{st-reduced-2} and the true dissimilarity scores $D(\mathbf{t}, \boldsymbol\xi)$ and $D(\mathbf{t}, S_{\boldsymbol\tau}\mathbf{m} + \boldsymbol\zeta)$ are computed.

All the fixed, constant parameters used in the simulations 
are selected to cover a reasonably wide range of possible applications, template structures, intensity distributions, and signal-to-noise ratios.

\subsection{Aggregation of the results}

The results of propositions \ref{st-distorted} and \ref{st-reduced-2} provide formulas for the expected values of the PWC nUV measure when the possible distortions have a particular cross-product structure. There can be numerous meaningful cross-product structures representing the possible distortions in various fields of applications. Picking any of them would deteriorate the generality of the experiments and due to space limitations, we can not examine many different structures in detail. Therefore in the test cases, almost all parameters of the templates, distortions, and noisy windows are being sampled. For the analysis, the computed dissimilarity scores are averaged for the entire population, which enables us to draw conclusions about the operation of the measure with distortions from many different cross-product structures. In the rest of the section, the averages of the expected values are denoted by $\overline{\mathbb{E}D(\mathbf{t},\mathbf{w})}$, and the averages of the computed dissimilarity scores are denoted by $\overline{D(\mathbf{t}, \mathbf{w})}$.

One can expect that for templates with varying sizes, varying numbers of bins might be ideal for template matching.
In order to compensate this variation in the sizes and enable the meaningful aggregation of the results, the numbers of bins for each test case are varied as follows. All the figures are computed for $2$ and $5$ bins, and for the number of bins determined by the Sturges-formula ($b_{st} = \lceil\log_2 d_{\boldsymbol\tau}\rceil + 1$), the Rice-rule ($b_{r} = \lceil 2 d_{\boldsymbol\tau}^{1/3}\rceil$) and the square root rule ($b_{sq} = \lceil \sqrt{d_{\boldsymbol\tau}}\rceil$), where $d_{\boldsymbol\tau}$ denotes the number of unique elements in $\mathbf{t}$. One can easily confirm that in the range of the experiments usually $5 \leq b_{st} \leq b_r \leq b_{sq}$ holds, therefore, we found it meaningful to plot the aggregated results in one figure and connect them with lines for a better visualization of trends, even though the values $b_{st}$, $b_{r}$ and $b_{sq}$ depend on the sizes of the templates.

\subsection{General distortions and greedy binning}

We have executed 5000 experiments of the test cases described above, and plotted the aggregated results in Figure \ref{res_greedy}, with the standard deviations denoted by vertical lines with a minor horizontal shift for visibility. As one can observe, the predictions of Propositions \ref{st-noise} and \ref{st-distorted} for the means of the distributions and the means of the real scores are very well aligned, with the highest relative difference of $0.03\%$ for the noise and $0.08\%$ for the distorted population. From these, we can conclude that despite the first-order approximation (subsection \ref{section32}), the formulae in propositions \ref{st-noise} and \ref{st-distorted} are close enough to the real values in the tested scenarios to expect an improvement in the separation power of PWC nUV by Theorem \ref{the-theorem} and Algorithm \ref{alg1}. 

The AUC scores for the various binning techniques are summarized in Figure \ref{res_greedy_b} and aggregated in Table \ref{res_kmeans2} with the matrix of p-values of the McNemar tests for the equality of the scores. 
One can observe that the greedy binning outperforms both EQW (by 13\% in aggregation) and EQF (by 26\% in aggregation) with statistical significance, providing a numerical validation for Theorem 1 in the scope of the experiments. Interestingly, although k-means binning is proved to be ideal when the distortion is spherically distributed around the template, it is outperforming EQW (by 4\%) and EQF (by 17\%) with statistical significance, indicating that there can be further configurations when k-means binning works well.
Comparing the greedy binning and the k-means binning, the most remarkable difference is that for the k-means binning there is no need for the estimation of the cross-product or covariance structure of the distortions. Consequently, the results suggest that even in the lack of any knowledge on the possible distortions, using k-means binning instead of EQW could improve the matching results with statistical significance.

\begin{figure}[t]
     \subfigure[\label{res_greedy_a}]{\includegraphics[width=0.5\textwidth]{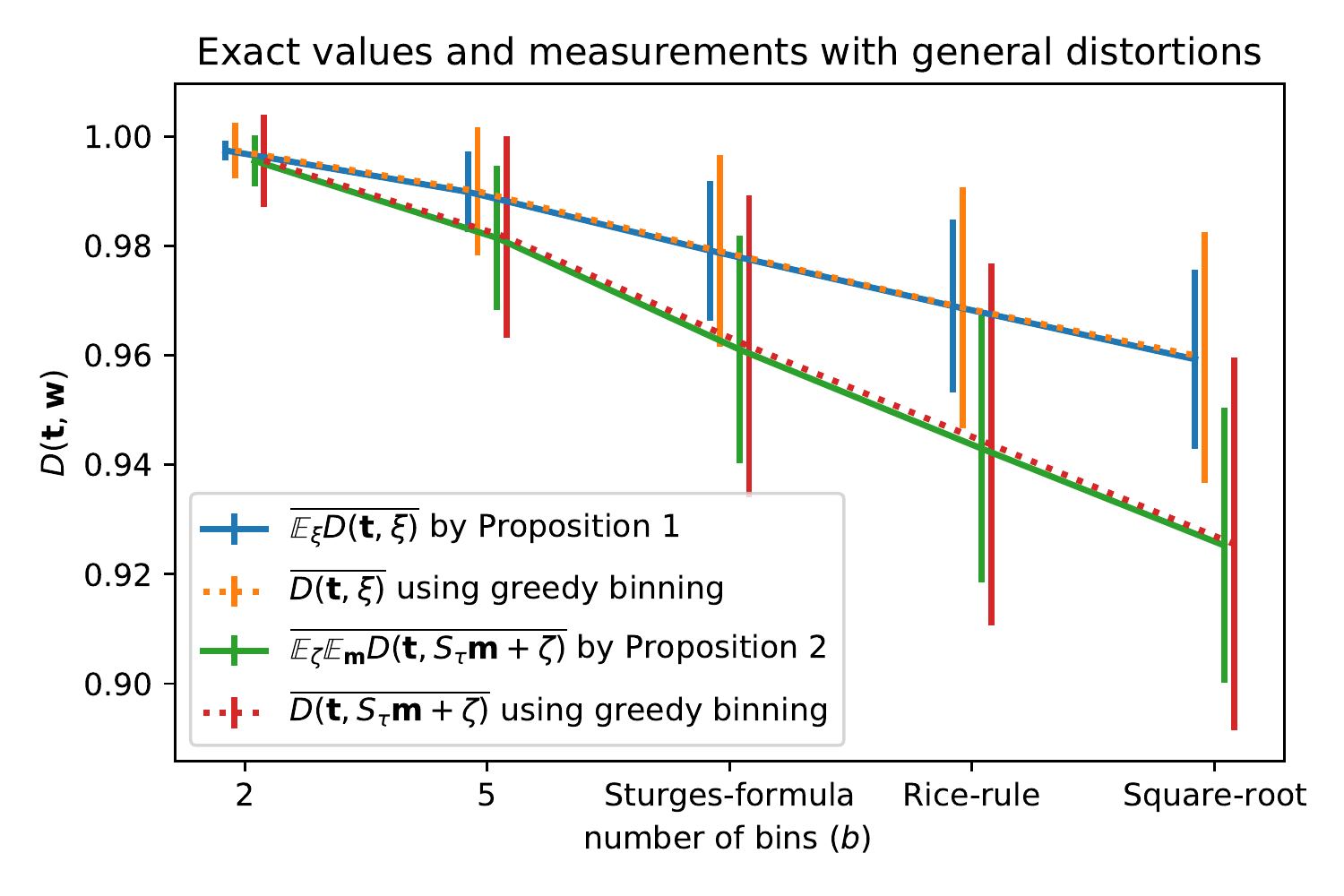}}
     \subfigure[\label{res_greedy_b}]{\includegraphics[width=0.5\textwidth]{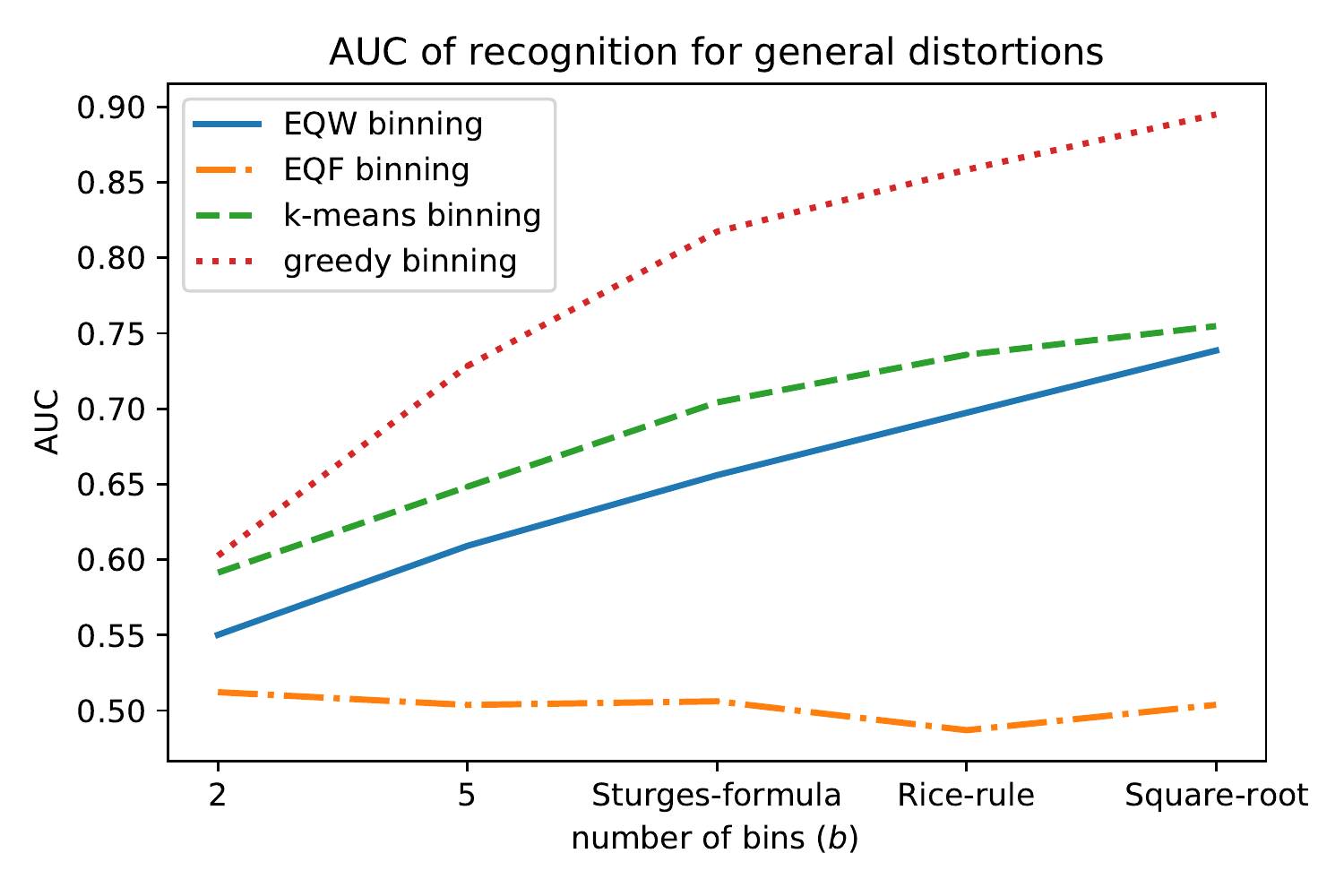}}
     
     \caption{Results for general distortions: fitting of the measurements to the theoretical values (a); AUC scores (b).}
     \label{res_greedy}
\end{figure}

\begin{table}
\begin{center}
\begin{scriptsize}

\begin{tabular}{l|rrrr|rrrr}
\toprule
& \multicolumn{4}{c}{General distortions} & \multicolumn{4}{c}{Spherical distortions}\\\hline
{} &      EQW &      EQF &  k-means &   greedy  &     EQW &     EQF &  k-means &  greedy\\
\midrule
EQW      &  1 & 0 &  0 & 0 & 1 & 0 &  2.3e-07 & 2.8e-04\\
EQF      & 0 &  1 & 0 &  0 & 0 & 1 &  0 & 0\\
k-means  &  0 & 0 &  1 & 0 & 2.3e-07 & 0 &  1 & 5.7e-02\\
greedy   & 0 &  0 & 0 &  1 & 2.8e-04 & 0 &  5.7e-02 & 1 \\\hline
AUC &  0.63 &  0.5 &  0.67 &  0.76 & 0.83 & 0.5 & 0.84 & 0.84 \\
\bottomrule
\end{tabular}

\end{scriptsize}
\end{center}
\caption{The matrix of p-values and the AUC scores for the various binning techniques with general and spherical distortions (values smaller than $1e-6$ are rounded to $0$.)}
\label{res_kmeans2}
\end{table}


\subsection{Spherical distortions}

Again, we have executed 5000 experiments of the test cases described and plotted the results in Figure \ref{res_kmeans}.
The predictions of propositions \ref{st-noise} and \ref{st-reduced-2} for the means of the distributions and the means of the real scores are very well aligned, with the highest relative difference of $0.03\%$ for the noisy windows and $0.1\%$ for the distorted population. Despite the first-order approximation, the results suggest that Proposition \ref{st-reduced-2} gives a good approximation of the expected value in the scope of the experiments. 
Comparing the AUC scores of recognition plotted in Figure \ref{res_kmeans_b} and aggregated in Table \ref{res_kmeans2} with the p-values of the McNemar tests on the equality of the scores, one can observe that the k-means and greedy techniques outperform EQW (by 1\% in aggregation) and EQF (by 34\% in aggregation) with statistical significance, however, the improvement is mainly for low numbers of bins, the performances quickly converge to that of EQW and the greedy technique (due to its suboptimality) gets below EQW in terms of AUC when the square root rule is applied to determine the number of bins. 
The reason for the limited improvement in the case of spherical distortions is that according to Proposition $\ref{st-reduced-2}$, k-means and the greedy techniques improve the discrimination power of PWC nUV by minimizing the term $\Vert A\mathbf{t} - \mathbf{t}\Vert$ which is a smaller decrease than minimizing both this term and the $- \langle A, SCov(\mathbf{m})S^T\rangle_F$ as pointed out in Proposition \ref{st-reduced-1}.


\begin{figure}[t]
     \subfigure[\label{res_kmeans_a}]{\includegraphics[width=0.5\textwidth]{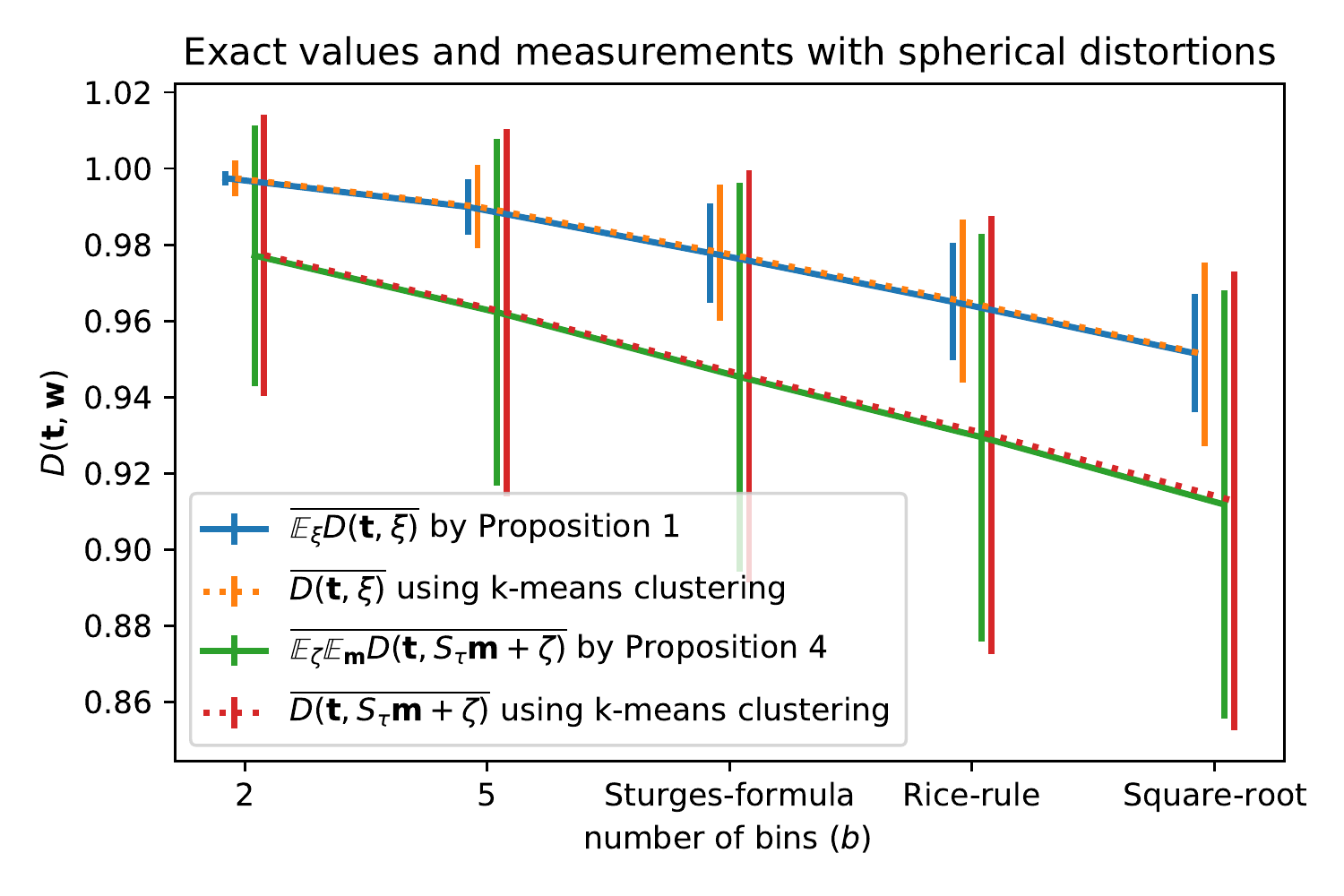}}
     \subfigure[\label{res_kmeans_b}]{\includegraphics[width=0.5\textwidth]{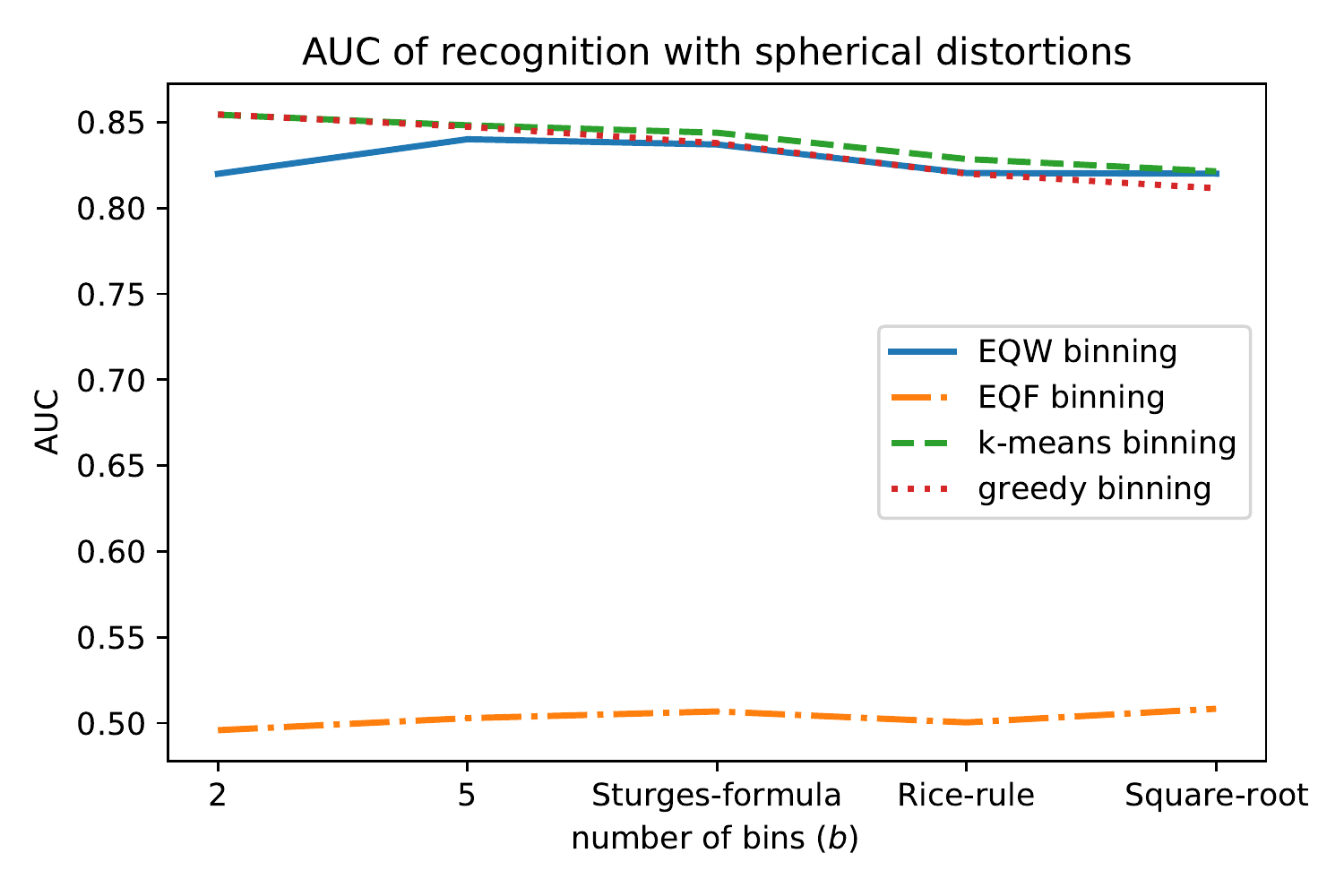}}
     
     \caption{Results for spherical distortions: fitting of the measurements to the theoretical values (a); AUC scores (b).}
     \label{res_kmeans}
\end{figure}

\subsection{A note on the low pattern recognition performance of EQF}
\label{low-eqf}
Interestingly, the AUC of EQF is 0.5 in both experiments, which means that it has no discriminative power in these settings. The operating principle of PWC nUV is that the slices describe the rough structure of the template as the values in the bins are close to each other. 
This assumption is definitely not satisfied by EQF, as having the same number of values in a bin completely neglects the structure of the template, can break many similar values into separate bins, providing a poor representation of the template. This phenomenon is a qualitative validation of our previous claim that binning techniques developed to reconstruct the empirical distribution function of a sample, do not necessarily perform well in other binning problems.

\section{Conclusions}
\label{secconc}

In this paper, we have examined the effect of binning strategies on the piecewise constant approximation of the normalized Unexplained Variance (nUV) (also known as MTM) dissimilarity measure. We defined the criterion of ideal operation in Definition \ref{def-optimality}, and managed to show in Theorem \ref{the-theorem} that the ideal binning needs to maximize the alignment of the projection matrix and the cross-product structure of the expected distortion. In order to obtain an approximate solution for this combinatorial optimization problem, we have proposed a greedy algorithm in Algorithm \ref{alg1}. In subsequent propositions, we have examined special cases of the general statement and arrived at the case of localized and spherically distributed distortions, for which the ideal binning can be determined by solving the k-means clustering problem according to Theorem \ref{th2}. 

In the Section \ref{secres} we have carried out experiments to see how much the simulation results are aligned with the first-order statistical approximations. According to the results, the relative error is less than 0.1\% in terms of the means of the figures. We also compared the performance of the proposed binning techniques to that of historical ones in pattern recognition scenarios and found the proposed approaches to outperform the historical ones by 13\% AUC in the case of general distortions and the greedy binning, and by 1\% AUC in the case of spherical distortions and k-means binning, in both cases with statistical significance.

The conclusions we can draw are summarized as follows.
Due to the analogies presented in Section \ref{sec-relation-to-mi}, nUV can be treated as a powerful alternative of MI, quantifying the uncertainty remaining about the window $\mathbf{w}$ given $\mathbf{t}$ in terms of variance. Thus, nUV is potentially applicable in any problem where MI is used as a similarity measure (template matching, registration, feature selection, etc.),
Although numerical experiments can never cover all the possible use cases of a general-purpose dissimilarity measure, due to the wide range of parameters used in the simulations, one can expect that using PWC nUV with the proposed binning techniques can improve its performance in terms of the AUC score.


\appendix

\section{Proof of Lemma \ref{lem-A}}
\label{prooflem1}
\begin{proof}
Due to the orthogonality of $S$, $S^TS$ is a diagonal matrix of type $\mathbb{Z}_+^{b\times b}$ with $(S^TS)_{ii}$ being equal to the cardinality of slice $i$. Inverting this matrix inverts the elements in the diagonal, with $(S^TS)^{-1}_{ii} = 1/\vert\mathcal{I}_i\vert$. Finally, due to the construction of $A$ and the orthogonality of $S$, one can readily see that in $A=S(S^TS)^{-1}S^T$, $A_{ij}$ is non-zero only if $i$ and $j$ fall in the same slice $\mathcal{I}_k$, and the value it takes is $1/\vert\mathcal{I}_k\vert$.
Due to the special structure of $A$, 
\begin{equation}
(A\mathbf{w})_i = \sum\limits_{j=1}^{d} A_{ij}\mathbf{w}_j = \sum\limits_{j\in\mathcal{I}_k}\dfrac{1}{\vert \mathcal{I}_k\vert} \mathbf{w}_j
\end{equation}
if $i\in\mathcal{I}_k$, which is the mean of elements of $\mathbf{w}$ in the slice $\mathcal{I}_k$ implied by $\mathbf{t}$.
\end{proof}
\section{Proof of Proposition \ref{st-noise}}
\label{proof-noise}
\begin{proof}
Most of the proofs in the paper are analogous, expanding the inner products in the expressions and simplifying them by utilizing the special properties of matrix $A$ highlighted in Lemma \ref{lem-A}. Due to space limitations, these steps are carried out only in this proof in all details.

According to subsection \ref{ratio-dist}, the numerator $\mathbb{E}_{\boldsymbol\xi}\mathbb{E}_{\mathbf{m}}\Vert A\mathbf{\boldsymbol\xi} - \boldsymbol\xi\Vert^2$ and the denominator $\mathbb{E}_{\boldsymbol\xi}\mathbb{E}_{\mathbf{m}}d\text{var}(\boldsymbol\xi)$ are evaluated separately.
The numerator is expanded as
\begin{equation}
\mathbb{E}_{\boldsymbol\xi} \Vert A\boldsymbol\xi - \boldsymbol\xi\Vert^2 = \mathbb{E}_{\boldsymbol\xi} [\langle A\boldsymbol\xi, A\boldsymbol\xi\rangle - 2\langle A\boldsymbol\xi, \boldsymbol\xi\rangle + \langle \boldsymbol\xi, \boldsymbol\xi \rangle].
\end{equation}
Evaluating the first term, utilizing Lemma \ref{lem-A} on the special properties of $A$, and the assumptions on the white noise ($0$ mean, finite $\sigma^2$ variance), one gets
\begin{equation}
\mathbb{E}_{\boldsymbol\xi}\left[\langle A\mathbf{\boldsymbol\xi}, A\mathbf{\boldsymbol\xi}\rangle\right]= \mathbb{E}_{\boldsymbol\xi}\left[\sum\limits_{i}\sum\limits_{j}\sum\limits_{k}A_{ij}A_{ik}\boldsymbol\xi_j\boldsymbol\xi_k\right]=\\\sum\limits_{i}\sum\limits_{j}A_{ij}^2\sigma^2=b\sigma^2.
\end{equation}
Similarly, $\mathbb{E}_{\boldsymbol\xi}\left[\langle A\boldsymbol\xi, \mathbf{t}\rangle\right]=b\sigma^2$ and $\mathbb{E}_{\boldsymbol\xi}\left\langle \boldsymbol\xi, \boldsymbol\xi\right\rangle=  d\sigma^2$.
For the denominator,
\begin{equation}
    \mathbb{E}_{\boldsymbol\xi} d\text{var}(\boldsymbol\xi)= \mathbb{E}_{\boldsymbol\xi} d\left( \dfrac{1}{d}\sum\limits_{i=1}^d \boldsymbol\xi_i^2 - \left(\dfrac{1}{d}\sum\limits_{i=1}^d\boldsymbol\xi_i\right)^2 \right) =  \sigma^2\left(d - 1\right).
\end{equation}
\end{proof}

\section{Proof of Proposition \ref{st-distorted}}
\label{proof-distorted}
\begin{proof}
First, we evaluate the numerator ($Num= \Vert A(S_{\boldsymbol\tau}\mathbf{m} + \boldsymbol\zeta) - (S_{\boldsymbol\tau}\mathbf{m} + \boldsymbol\zeta)\Vert^2$):
Expanding the inner product and carrying out the integration for $\boldsymbol\zeta$ (zero-mean white noise with finite variance $\sigma^2$) leaves the following non-zero terms.
\begin{equation}
\label{tmp000}
\mathbb{E}_{\boldsymbol\zeta}\mathbb{E}_\mathbf{m} Num = \mathbb{E}_\mathbf{m} \left(\Vert AS_{\boldsymbol\tau}\mathbf{m}\Vert^2 + \Vert S_{\boldsymbol\tau}\mathbf{m}\Vert^2 - 2\langle AS_{\boldsymbol\tau}\mathbf{m}, S_{\boldsymbol\tau}\mathbf{m}\rangle + \sigma^2(d - b)\right)
\end{equation}
Due to the idempotence of $A$, $\langle AS_{\boldsymbol\tau}\mathbf{m}, S_{\boldsymbol\tau}\rangle = \langle AS_{\boldsymbol\tau}\mathbf{m}, AS_{\boldsymbol\tau}\mathbf{m}\rangle$, thus, 
\begin{equation}
\mathbb{E}_{\boldsymbol\zeta}\mathbb{E}_\mathbf{m} Num = \mathbb{E}_\mathbf{m} \left(\Vert S_{\boldsymbol\tau}\mathbf{m}\Vert^2 - \langle AS_{\boldsymbol\tau}\mathbf{m}, S_{\boldsymbol\tau}\mathbf{m}\rangle + \sigma^2(d - b)\right)
\end{equation}
Let $\text{Cross}(\mathbf{m})_{ij} = \mathbb{E}_\mathbf{m}\mathbf{m}_i\mathbf{m}_j$. Carrying out the integration for $\mathbf{m}$,
\begin{equation}
\mathbb{E}_{\boldsymbol\zeta}\mathbb{E}_\mathbf{m} Num = \langle \mathbf{n}_{\boldsymbol\tau}, \mathbb{E}_\mathbf{m}\mathbf{m}^2\rangle - \langle A, S_{\boldsymbol\tau} \text{Cross}(\mathbf{m}) S_{\boldsymbol\tau}^T\rangle_{F} + \sigma^2(d - b),
\end{equation}
for the expectation of the numerator, where $\langle, \rangle_F$ denotes the Frobenius inner product.
Similarly for the denominator, 
utilizing the special properties of $S_{\boldsymbol\tau}$:
\begin{multline}
\mathbb{E}_{\boldsymbol\zeta}\mathbb{E}_\mathbf{m} d\text{var}(S_{\boldsymbol\tau}\mathbf{m} + \boldsymbol\zeta) = \mathbb{E}_{\boldsymbol\zeta}\mathbb{E}_\mathbf{m}d\left[\text{var}(S_{\boldsymbol\tau}\mathbf{m}) + \text{var}(\boldsymbol\zeta)\right]=\\ 
\langle \mathbf{n}_{\boldsymbol\tau}, \mathbb{E}_\mathbf{m}\mathbf{m}^2\rangle - \dfrac{1}{d}\mathbf{n}_{\boldsymbol\tau} Cross(\mathbf{m}) \mathbf{n}_{\boldsymbol\tau}^T + \sigma^2\left(d - 1\right).
\end{multline}
\end{proof}

\section{Proof of Proposition \ref{st-reduced-1}}
\label{proof-reduced-1}
\begin{proof}
$\mathbf{E}_\mathbf{m}\mathbf{m} = \boldsymbol\tau$ implies $\mathbf{E}_\mathbf{m}S_{\boldsymbol\tau}\mathbf{m} = \mathbf{t}$. One can change the variable of integration from $\mathbf{m}$ to $\mathbf{m}' = \mathbf{m} - \boldsymbol\tau$ with $\mathbb{E}_\mathbf{m} \mathbf{m}' = 0$. Substituting $\mathbf{m} = \mathbf{m}' + \boldsymbol\tau$ into the first two terms of the numerator of (\ref{the-solution}),
\begin{align}
\langle \mathbf{n}_{\boldsymbol\tau}, \mathbb{E}_\mathbf{m}\left[(\mathbf{m}' + {\boldsymbol\tau})^2\right]\rangle &= \langle \mathbf{n}_{\boldsymbol\tau}, \mathbb{E}_\mathbf{m} \mathbf{m}'^2\rangle + \langle \mathbf{t}, \mathbf{t}\rangle,\\
\langle A, S_{\boldsymbol\tau} Cross(\mathbf{m}' + \boldsymbol\tau)S_{\boldsymbol\tau}^T\rangle_F &= \langle A, S_{\boldsymbol\tau} Cov(\mathbf{m}')S_{\boldsymbol\tau}^T\rangle_F + \langle A\mathbf{t}, \mathbf{t}\rangle.
\end{align}
The last terms ($\langle \mathbf{t}, \mathbf{t}\rangle$, $\langle A\mathbf{t}, \mathbf{t}\rangle$) can be transformed to $\Vert A\mathbf{t} - \mathbf{t}\Vert^2$, by adding and subtracting $\langle A\mathbf{t}, \mathbf{t}\rangle$, and utilizing that $\langle A\mathbf{t}, \mathbf{t}\rangle = \langle A\mathbf{t}, A\mathbf{t}\rangle$, due to the idempotence of $A$.
For the denominator,
\begin{multline}
\mathbb{E}_{\boldsymbol\zeta}\mathbb{E}_\mathbf{m}d\text{var}(S_{\boldsymbol\tau}\mathbf{m} + \boldsymbol\zeta) = \mathbb{E}_{\boldsymbol\zeta}\mathbb{E}_\mathbf{m}d\text{var}(\mathbf{t}+S_{\boldsymbol\tau}\mathbf{m}' + \boldsymbol\zeta)=\\
d\text{var}(\mathbf{t}) + \mathbb{E}_{\boldsymbol\zeta}\mathbb{E}_\mathbf{m}\left[d\text{var}(S_{\boldsymbol\tau}\mathbf{m}') + d\text{var}(\boldsymbol\zeta))\right].
\end{multline}
Utilizing the results of Proposition \ref{st-distorted} and $Cross(\mathbf{m}') = Cov(\mathbf{m}')$ , one gets
\begin{equation}
    \mathbb{E}_{\boldsymbol\zeta}\mathbb{E}_\mathbf{m}d\text{var}(S_{\boldsymbol\tau}\mathbf{m} + \boldsymbol\zeta)= \\ d\text{var}(\mathbf{t}) + \langle \mathbf{n}_{\boldsymbol\tau}, \mathbb{E}_\mathbf{m}\mathbf{m'}^2\rangle - \dfrac{1}{d}\mathbf{n}_{\boldsymbol\tau} Cov(\mathbf{m'}) \mathbf{n}_{\boldsymbol\tau}^T + \sigma^2\left(d - 1\right).
\end{equation}

\end{proof}


\section*{References}




\bibliographystyle{model2-names}
\bibliography{refs}

%

\end{document}